\documentclass[lettersize, journal, twoside]{IEEEtran}

\IEEEoverridecommandlockouts                              

\usepackage{xr}

\newcommand{\omitted}[1]{}

\usepackage{cite}

\usepackage{float}
\usepackage{comment}
\usepackage{siunitx}
\usepackage{relsize}
\usepackage{ifthen}
\usepackage[colorinlistoftodos]{todonotes}

\usepackage[caption=false]{subfig}

\usepackage[vlined,ruled,linesnumbered]{algorithm2e}
\usepackage{graphics} 
\usepackage{rotating}
\usepackage{color}
\usepackage{enumerate}
\usepackage[T1]{fontenc}
\usepackage{psfrag}
\usepackage{epsfig} 
\usepackage{booktabs}
\usepackage{graphicx,url}
\usepackage{multirow}
\usepackage{array}
\usepackage{latexsym}
\usepackage{amsfonts}
\usepackage{amsmath}
\usepackage{amssymb}
\usepackage{xstring}
\usepackage{algorithmic}
\usepackage{multirow}
\usepackage{xcolor}
\usepackage{prettyref}
\usepackage{flexisym}
\usepackage{bigdelim}
\usepackage{breqn} 
\usepackage{listings}
\let\labelindent\relax
\usepackage{xspace}
\usepackage{bm}
\graphicspath{{./figures/}}
\usepackage{tikz}
\usetikzlibrary{matrix,calc}
\usepackage{lipsum}
\usepackage{mdwlist}

\let\itemize\undefined
\makecompactlist{itemize}{stditemize}
\usepackage{enumitem}
\usepackage{caption}
\usepackage{amsthm}
\usepackage{mathtools}
\usepackage{cellspace}

\makeatletter
\let\NAT@parse\undefined
\makeatother
\usepackage{hyperref}
\usepackage{cleveref}

\usepackage{tabularray}

\hypersetup{%
  colorlinks=true,
  linkcolor=blue,
  filecolor=magenta,      
  urlcolor=red,
  citecolor=red,
  linkbordercolor={0 0 1}
}

\newtheorem{theorem}{Theorem}
\newtheorem{problem}{Problem}

\newtheorem{corollary}{Corollary}

\newtheorem{definition}{Definition}
\newtheorem{proposition}{Proposition}

\newtheorem{remark}{Remark}
\newtheorem{example}{Example}

\newcommand{\bdmath}{\begin{dmath}}
\newcommand{\edmath}{\end{dmath}}
\newcommand{\beq}{\begin{equation}}
\newcommand{\eeq}{\end{equation}}
\newcommand{\bdm}{\begin{displaymath}}
\newcommand{\edm}{\end{displaymath}}
\newcommand{\bea}{\begin{eqnarray}}
\newcommand{\eea}{\end{eqnarray}}
\newcommand{\beal}{\beq \begin{array}{lll}}
\newcommand{\eeal}{\end{array} \eeq}
\newcommand{\beas}{\begin{eqnarray*}}
\newcommand{\eeas}{\end{eqnarray*}}
\newcommand{\ba}{\begin{array}}
\newcommand{\ea}{\end{array}}
\newcommand{\bit}{\begin{itemize}}
\newcommand{\eit}{\end{itemize}}
\newcommand{\ben}{\begin{enumerate}}
\newcommand{\een}{\end{enumerate}}

\newcommand{\calA}{{\cal A}}
\newcommand{\calB}{{\cal B}}
\newcommand{\calC}{{\cal C}}

\newcommand{\calE}{{\cal E}}

\newcommand{\calG}{{\cal G}}

\newcommand{\calI}{{\cal I}}

\newcommand{\calN}{{\cal N}}

\newcommand{\calS}{{\cal S}}

\newcommand{\calV}{{\cal V}}

\newcommand{\calX}{{\cal X}}

\definecolor{myblue}{RGB}{65 105 225}

\newcommand{\hide}[1]{}

\newcommand{\hiddenText}{{\color{gray} hidden text.}}
\newcommand{\hideWithText}[1]{\hiddenText}

\newcommand{\opt}{^{\star}}

\newcommand{\scenario}[1]{{\smaller \sf#1}\xspace}

\newcommand{\ie}{\emph{i.e.},\xspace}
\newcommand{\eg}{\emph{e.g.},\xspace}
\newcommand{\myin}{\, \in \,}
\newcommand{\red}[1]{{\color{red}#1}}
\newcommand{\blue}[1]{{\color{blue}#1}}

\newcommand{\alg}{\scenario{RAG}}
\newcommand{\sg}{\scenario{SG}}
\newcommand{\dsm}{\scenario{DSM}}

\newcommand{\myParagraph}[1]{{\bf #1.}\xspace}

\renewcommand{\opt}{\scenario{OPT}}

\newcommand{\ourcurv}{\scenario{coin}_{f,i}}
\newcommand{\ourcurvnew}{\scenario{coin}_{f}}

\newcommand{\ouragent}{\mathcal{N}}
\newcommand{\oursol}{\calA^{\alg}}
\newcommand{\singlesol}{a^{\alg}}

\newcommand{\rob}{j}

\newcommand{\agentselected}{\calI}
\newcommand{\solutionselected}{\calA}

\PassOptionsToPackage{end}{algorithmic}

\title{%
\fontsize{23}{23} \selectfont
Communication- and Computation-Efficient Distributed {Submodular Optimization in Robot Mesh Networks} 
}
\author{Zirui Xu,$^{\dagger,\star}$,~\IEEEmembership{Graduate Student Member,~IEEE,} Sandilya Sai Garimella,$^{\ddagger,\star}$,~\IEEEmembership{Graduate Student Member,~IEEE,} Vasileios Tzoumas$^\dagger$,~\IEEEmembership{Senior Member,~IEEE}
	\thanks{$^\star$Equal contribution.}
 \thanks{
 $^\dagger$Department of Aerospace Engineering, University of Michigan, Ann Arbor, MI 48109 USA;  {\tt\footnotesize \{ziruixu,vtzoumas\}@umich.edu}} 
 	\thanks{$^{\ddagger}$Institute for Robotics and Intelligent Machines, Georgia Institute of Technology, Atlanta, GA 30332 USA; {\tt\footnotesize sgarimella34@gatech.edu}. When part of the work was completed, the author was with the Department of Robotics, University of Michigan, Ann Arbor, MI 48109 USA.} 
    \thanks{This work was partially supported by NSF CAREER No. 2337412.}
}

\begin{document}

\maketitle

\begin{abstract}
We provide a communication- and computation-effi- cient method for distributed submodular optimization in robot mesh networks.
Submodularity is a property of diminishing re- turns that arises in active information gathering such as mapping, surveillance, and target tracking.
Our method, Resource-Aware distributed Greedy (\alg), introduces a new distributed optimization paradigm that enables scalable and near-optimal action coordination.  
To this end, \alg requires each robot to make decisions based only on information received from and about their neighbors. 
In contrast, the current paradigms allow the relay of information about all robots across the network. 
As a result, \alg's decision-time scales linearly with the network size, while state-of-the-art near-optimal submodular optimization algorithms scale cubically.  
We also characterize how the designed mesh-network topology
affects \alg's approximation performance. 
Our analysis implies that sparser networks favor scalability without
proportionally compromising approximation performance: while \alg's decision time scales linearly with network size, the gain in approximation performance scales sublinearly.  
We demonstrate \alg's performance in simulated scenarios of area detection with up to 45 robots, simulating realistic robot-to-robot (r2r) communication speeds such as the $0.25$ Mbps speed of the Digi XBee 3 Zigbee 3.0.  In the simulations, \alg enables real-time planning, up to three orders of magnitude faster than competitive near-optimal algorithms, while also achieving superior mean coverage performance.
To enable the simulations, we extend the high-fidelity and photo-realistic simulator AirSim by integrating a scalable collaborative autonomy pipeline to tens of robots and simulating r2r communication delays. Our code is available at {\url{https://github.com/UM-iRaL/Resource-Aware-Coordination-AirSim}}.
\end{abstract}

\vspace{-2mm}

\section{Introduction}

\IEEEPARstart{I}{n} the future, numerous distributed robots will be coordinating actions via  \textit{robot-to-robot} (r2r) communication to execute {information-gathering tasks} such as collaborative mapping~\cite{atanasov2015decentralized}, surveillance~\cite{dames2017detecting}, and target tracking~\cite{corah2021scalable}.  Such tasks require efficiency (scalability) and effectiveness (optimality), especially in crowded scenarios where the robots are many and operate close to each other; see, for example, the semantic-driven task of road detection with 45 aerial robots in Fig.~\ref{fig:intro}.

\begin{figure}[H]
    \captionsetup{font=footnotesize}
    \begin{minipage}{\columnwidth}
        \includegraphics[width=1\columnwidth]{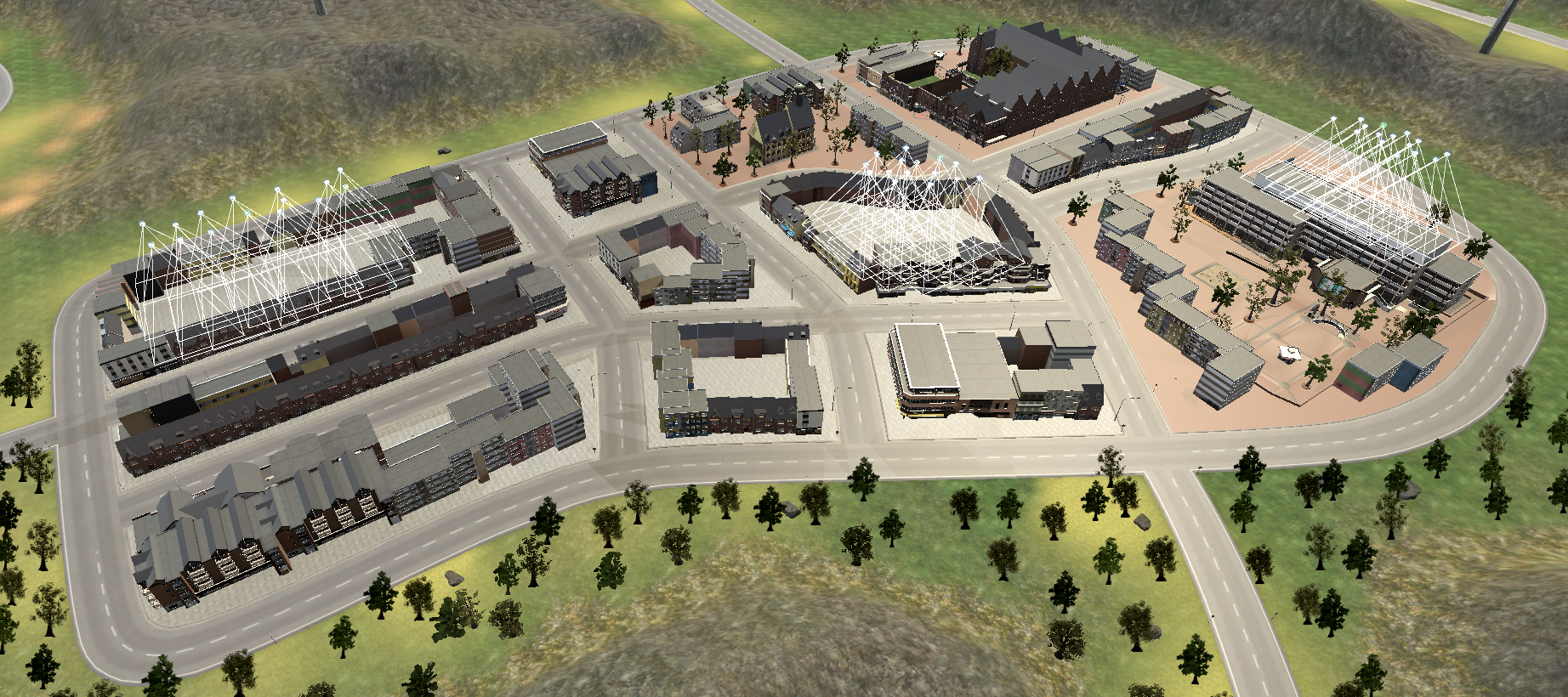}
        \caption*{(a) The drones start from their initial deployment points.
        }
        \label{fig:simulator-1}
    \end{minipage} 
    \vspace{2mm} \\
    \begin{minipage}{\columnwidth}
        \includegraphics[width=1\columnwidth]{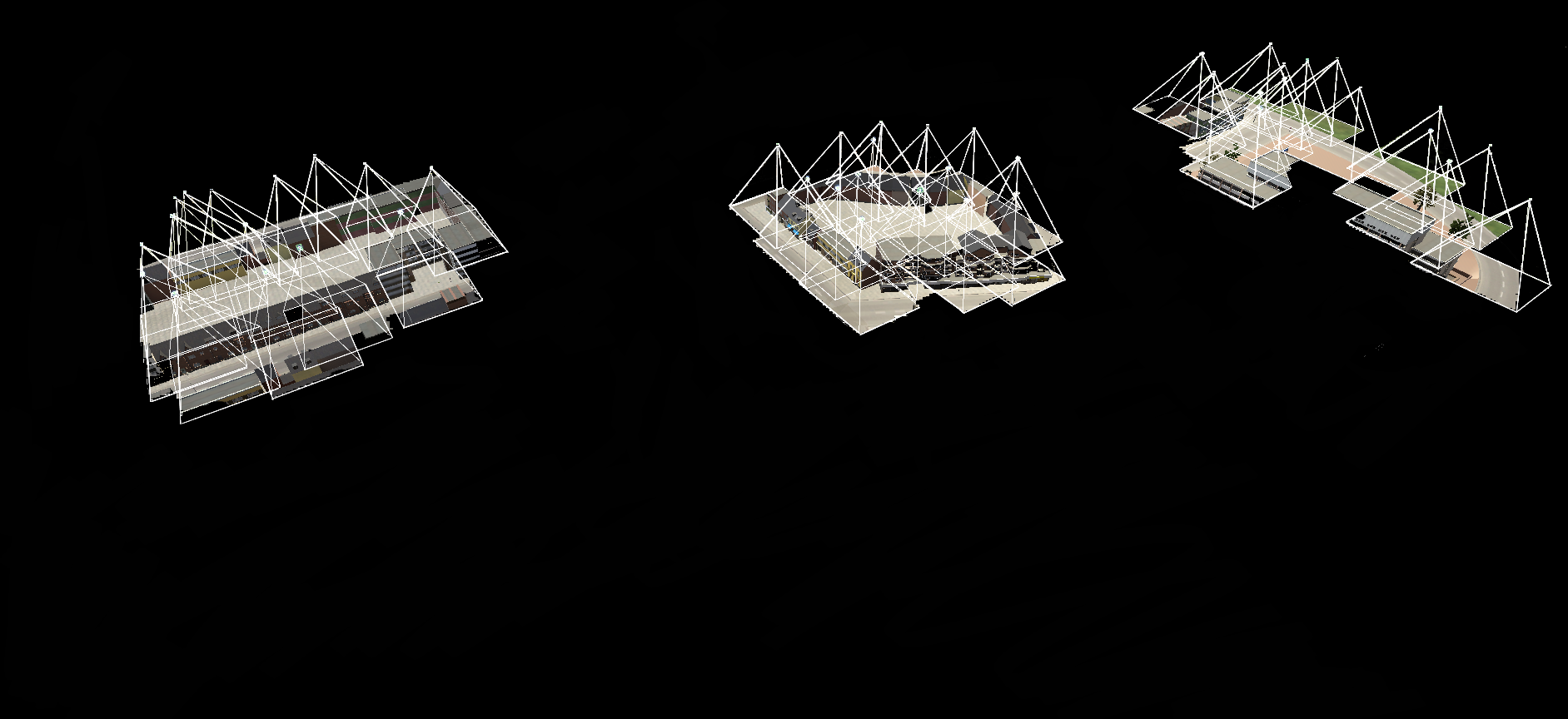}
        \caption*{(b) Progress after 3 replanning steps.}
        \label{fig:simulator-2}
    \end{minipage}
    \vspace{2mm} \\
    \begin{minipage}{\columnwidth}
        \includegraphics[width=1\columnwidth]{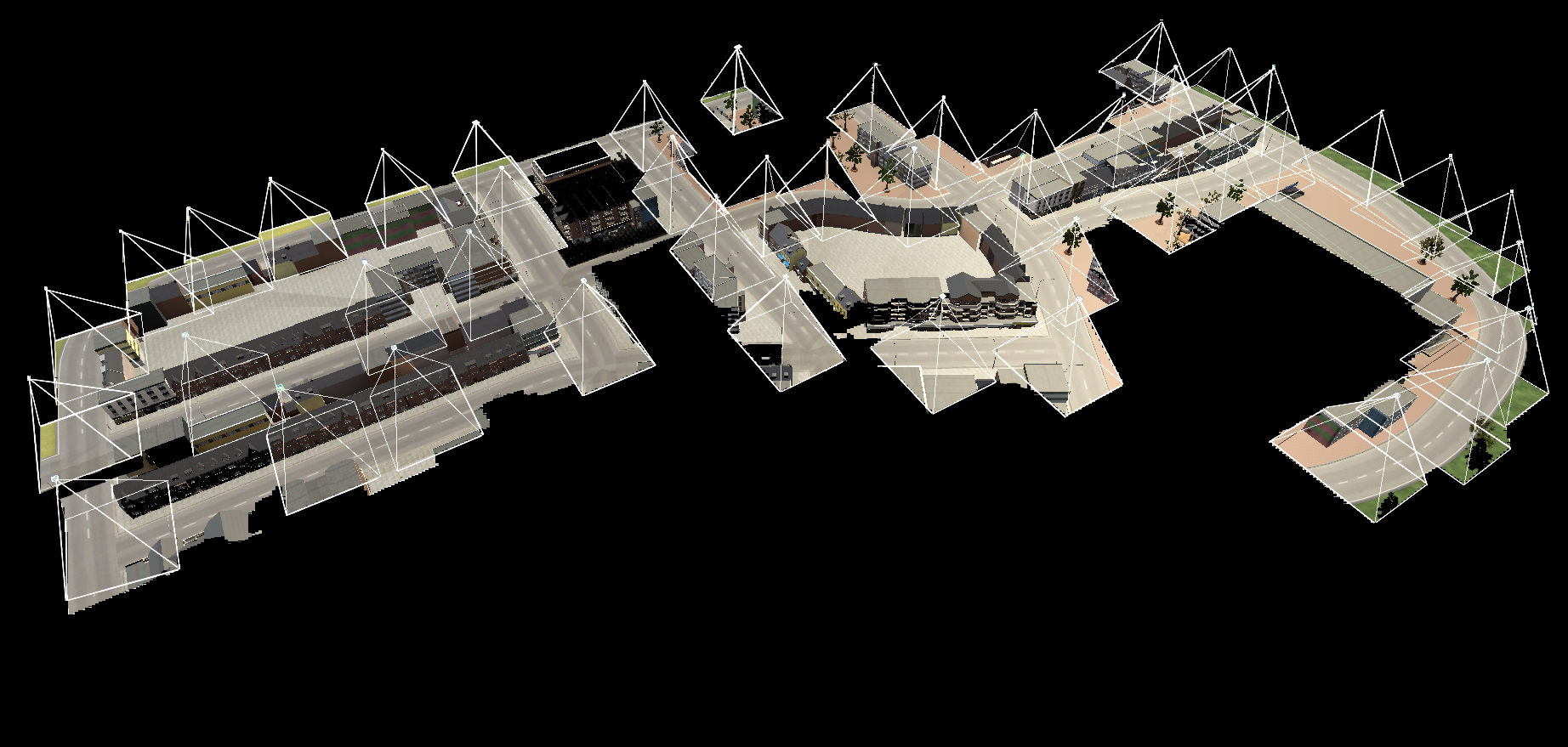}
        \caption*{(c) Progress after 8 replanning steps.
        }
        \label{fig:simulator-3}
    \end{minipage}
    \caption{\textbf{Scenario of Collaborative Road Detection with 45 Aerial Robots}. 
    The drones are deployed in an unknown environment and tasked to collaboratively detect the roads and visually cover them.  Particularly, the drones aim to collectively maximize the total covered road area in their top-view field of view (FOV). 
    At each action coordination step, each drone chooses a few other drones to coordinate with, subject to its onboard communication and computation bandwidth constraints.  The full collaborative autonomy pipeline is depicted in Fig.~\ref{fig:pipeline}. (a)--(c) depict the task progress across time: (a) The drones start from their initial deployment points; (b) Progress after 3 replanning steps; (c) The drones have achieved maximal road coverage.
    }
    \label{fig:intro}
\end{figure}

{But achieving scalability and optimality is hard.  The first reason is that such tasks take the form of distributed submodular optimization, an NP-hard combinatorial optimization pro-}

\begin{figure*}[t]
\captionsetup{font=footnotesize}
\begin{center}
\hspace{8mm}\includegraphics[width=\textwidth]{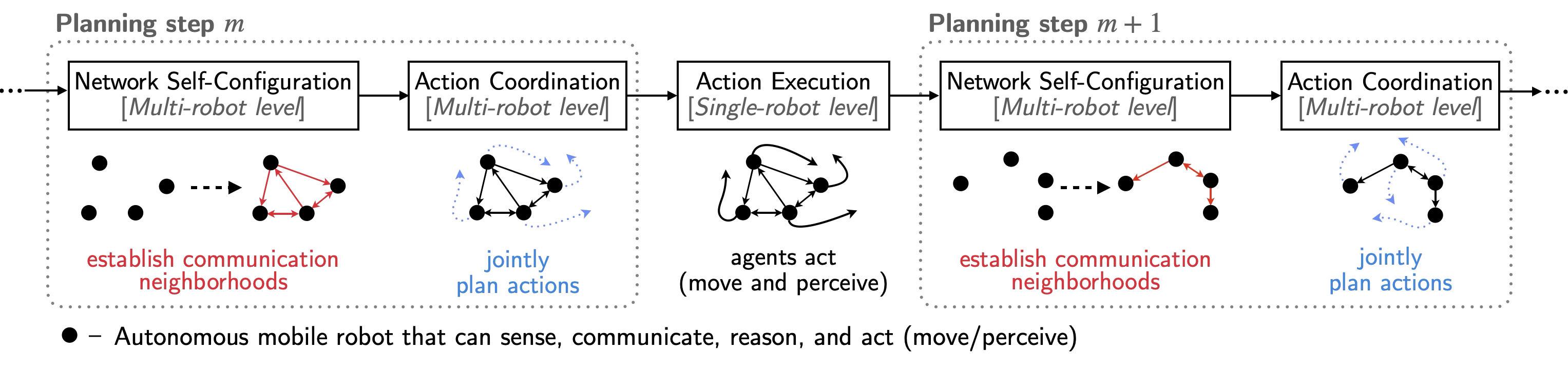}
\end{center}
\vspace{-2mm}
\caption{\textbf{Pipeline of Collaborative Mobile Autonomy}. The robots sequentially perform over a mesh network: (a) \textit{Network Self-Configuration {Step}}: Given the observed environment and state of the robots, the robots decide with which other robots to communicate with, subject to their onboard bandwidth constraints; and (b) \textit{Action Coordination {Step}}: The robots jointly plan actions ---how to move and sense the environment--- upon coordinating over the established communication network. (c) \textit{Action Execution {Step}}:
The robots execute their selected actions and perceive the environment. In this paper, our focus is on Action Coordination over tasks that take the form of submodular optimization.  We present a communication- and computation-efficient distributed algorithm that scales linearly with the network size, in contrast to the cubic time complexity of the competitive near-optimal algorithms. 
Along with our algorithmic contribution, we provide a rigorous analysis of how each agent's communication neighborhood affects the near-optimality of the optimization. {The analysis implies that establishing sparser neighborhoods favors scalability without
proportionally compromising approximation performance.}
}\label{fig:pipeline}
\end{figure*}

\noindent{blem~\cite{Feige:1998:TLN:285055.285059}. 
Therefore, such tasks require increased computations and communications to be solved optimally.  Submodularity is a property of diminishing returns, and is encountered across robotics~\cite{sung2023survey} as well as machine learning~\cite{bilmes2022submodularity} and control~\cite{clark2016submodularity}.  It captures the intuition that when any two robots are nearby and collect same information, \eg track the same targets, then one of the robots is redundant (it would be more effective if the robots were collecting different information).  On top of the above reason, another fundamental reason is the communication and computation limitations of the robots versus the r2r messaging requirements by the tasks.  For example, in active distributed simultaneous localization and mapping (SLAM), the inter-robot communication messages can grow to a few MBs within a few hundred meters of explored environment~\cite{liu2024slideslam} ($1$ MB $=8$ Mb).  But the wireless r2r communication speeds range from 0.25 Mbps, \eg achieved by the Digi XBee 3 Zigbee 3.0 antenna module, to 100 Mbps, \eg achieved by the Silvus Tech SL5200 antenna module.  Therefore, transmitting even a single message can take seconds.
}

{The current approaches are either near-optimal but not scalable, or real-time but offer no performance guarantees. The near-optimal methods~\cite{sviridenko2017optimal,fisher1978analysis} that are being used for active information gathering with multiple robots~\cite{krause2008near,dames2017detecting,singh2009efficient,tokekar2014multi,atanasov2015decentralized,corah2019distributed,lauri2020multi,schlotfeldt2021resilient,biggs2022non,akcin2023fleet,cai2023energy}, scale cubically with the network size, resulting in impractical running times in mesh networks with tens of robots (up to tens of minutes per action coordination step in the presence of real-world communication delays). The state-of-the-art multi-robot motion planning methods are real-time~\cite{smith2018distributed,yu2021smmr,gao2022meeting,zhou2023racer} but they are suitable in spatial exploration where robots are expected to operate far from each other. Then, the submodular structure can be ignored and opportunistic coordination is sufficient.   
We discuss the related work in extension in \Cref{sec:related}.}

\begin{figure*}[ht]
    \captionsetup{font=footnotesize}
    \centering
    \includegraphics[width=1\textwidth]{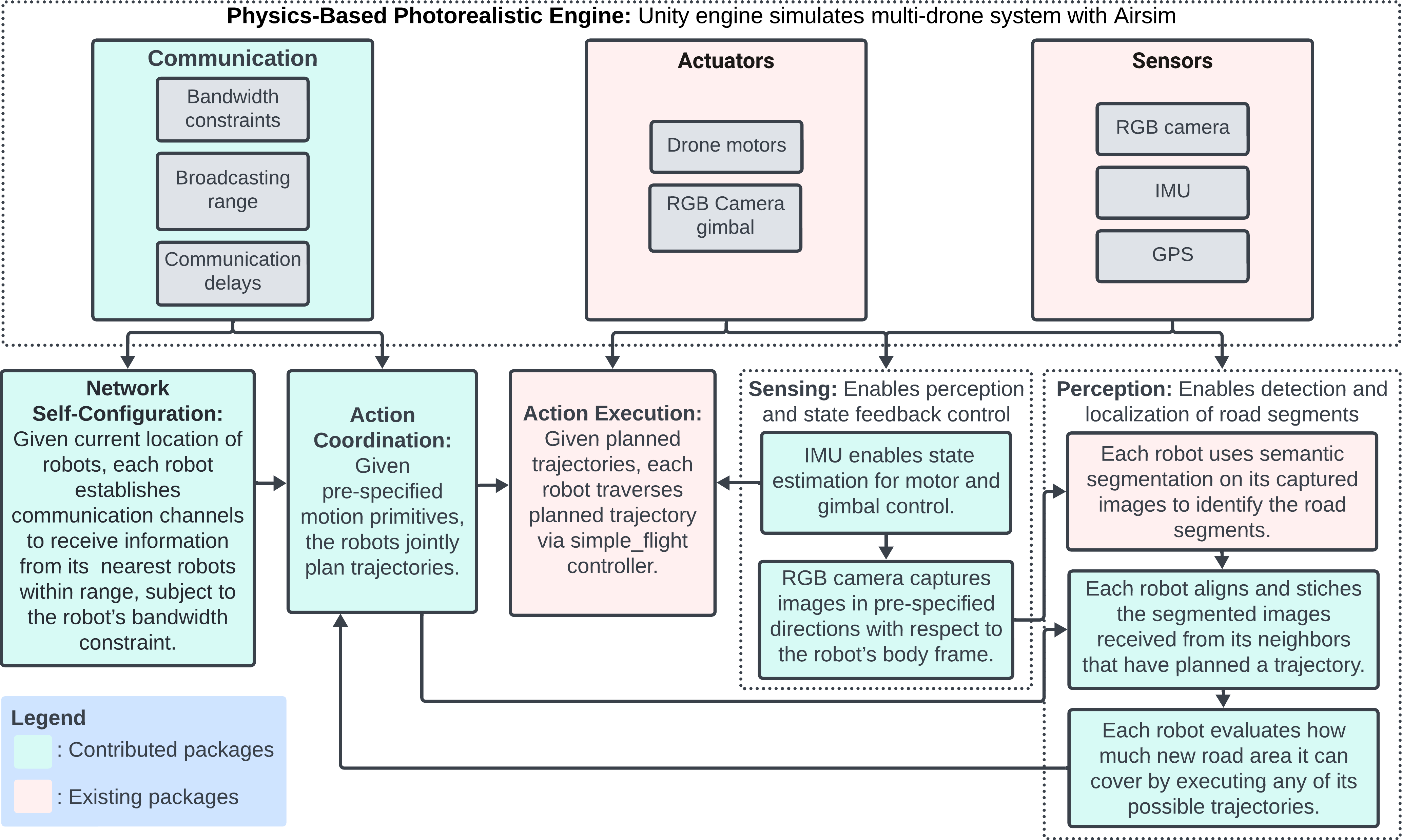}
        \vspace{-3mm}
    \caption{\textbf{Simulator pipeline.}  We provide a high-fidelity simulator that extends AirSim~\cite{shah2018airsim} to the multi-robot setting by integrating the autonomy pipeline of Fig.~\ref{fig:pipeline} and simulating r2r communication delays among other communication constraints.  The illustrated pipeline is customized for the action coordination algorithm provided herein.  The pipeline can be modified to other algorithms as desired.  
    } 
    \label{fig:simulator}
\end{figure*}

{In this paper, instead, we focus on crowded scenarios where robots operate close to one another, and their objective is not merely exploration; it is discovery and/or monitoring of specific information, \eg landmarks, mobile targets, areas and objects with specific semantics~\cite{atanasov2015decentralized,hughes2024multi,sung2023survey} (Fig.~\ref{fig:intro}). Then, the submodular structure is present. Thus, to jointly maximize the submodular objective, the robots must intelligently coordinate actions to enable both scalability and near-optimality.}

\myParagraph{Contributions} {We provide a communication- and computa- tion-efficient method for distributed submodular optimization in robot mesh networks.  The method, \textit{Resource-Aware distributed Greedy} (\alg), enables both scalable and near-optimal coordination over robot mesh networks despite real-world communication and computational delays.  In contrast, the current approaches are either near-optimal but not scalable or are real-time but offer no performance guarantees. 
\alg applies to any distributed submodular optimization task.  In this paper, we evaluate \alg in a semantics-driven area exploration task of road detection and coverage with tens of robots (Fig.~\ref{fig:intro}), via high-fidelity simulations. In more detail, our technical and simulator contributions are as follows.

We characterize \alg's time complexity and approximation performance
as a function of the state of the robot mesh network such as the robots' communication neighborhoods, task objective function, and  expected delays due to the robots' computation and communication capabilities:
}

\begin{itemize}[leftmargin=*]
\item {\textit{Time complexity:} \alg's decision-time scales linearly with the network size. Instead, the competitive near-optimal submodular optimization algorithms scale cubically with the network size.  To this end, \alg introduces a resource-minimal distributed optimization paradigm that requires each robot to make decisions based only on locally observed information and information received from and about their neighbors only.  In contrast, the current submodular optimization paradigms~\cite{sviridenko2017optimal,fisher1978analysis}, widely applied in robotics and control~\cite{krause2008near,dames2017detecting,singh2009efficient,tokekar2014multi,atanasov2015decentralized,corah2019distributed,lauri2020multi,schlotfeldt2021resilient,biggs2022non,akcin2023fleet,cai2023energy},  that instead require the relay of information about all robots across the network.}

\item {\textit{Approximation performance:} We present suboptimality bounds for \alg.
The bounds are useful as follows: (i) They capture the approximation performance of \alg as a function of each agent's local mesh-network topology, quantifying how denser mesh networks (more coordination) improve the approximation performance. 
(ii) The bounds inform how to design each agent's communication neighborhood to balance the trade-off between decision speed and optimality.  Particularly, we prove that the approximation performance gain grows sublinearly with the agents' communication neighborhood size. Therefore, in combination with the fact that \alg's decision time scales linearly with the network size, the bounds imply that sparser neighborhoods would favor decision speed without proportionally compromising approximation performance.  The appropriate level of neighborhood sparsity that balances the trade-off for the task at hand can be found via experimentation in high-fidelity simulations, as we illustrate in our simulations.}
\end{itemize}

{To perform high-fidelity evaluations, we enhance the physics-based and photo-realistic simulator  AirSim~\cite{shah2018airsim} to enable scalable simulation of robot mesh networks, including simulation of r2r communication delays (Fig.~\ref{fig:simulator}).}
{To this end, we first implement multi-threading for sensor data processing in ROS wrappers, resulting in a tenfold improvement in data frequency across multiple sensing modalities, including RGB, depth, segmentation, IMU, GPS, and barometer.
Second, we develop an r2r coordination module that models realistic communication delays and ranges between drones, with signal attenuation and dynamic neighborhood relationships managed through the ROS master.
Finally, we implemented automatons for capturing segmentation images via rotating gimbals and calculating marginal information gain by stitching semantic data from multiple drones.  The implementation uses ROS1, thus inherently incorporating its transport layer delays and message passing overhead, as shown in Figs.~\ref{fig:combined-coverage-analysis-15robots-100Mbps}--\ref{fig:combined-coverage-analysis-45robots-0p25Mbps}.} 

{{Our code is available at {\url{https://github.com/UM-iRaL/Resource-Aware-Coordination-AirSim}}.}}

\myParagraph{Evaluations} {We evaluate \alg in the provided simulator across four  scenarios of road detection and coverage (Fig.~\ref{fig:intro}). The scenarios include two team sizes, 15 and 45 robots, and two communication rates, 0.25 Mbps, \eg achieved by the Digi XBee 3 Zigbee 3.0 antenna module, and 100 Mbps, \eg achieved by the Silvus Tech SL5200 antenna module. 

\alg achieves real-time planning, up to three orders of magnitude faster than competitive near-optimal submodular optimization algorithms; and, when the robots maintain a neighborhood size $\geq 2$, \alg also achieves superior mean coverage.  The experiments also demonstrate \alg scales linearly with the network size, as predicted by our theoretical analysis: from the 15-robot to the 45-robot case,  \alg scales linearly, being at most 3x slower compared to the 15-robot case.  In contrast, the compared near-optimal algorithms scale cubically, being up to 60x slower compared to the 15-robot case.}

\myParagraph{Organization of the Remaining Paper} In \Cref{sec:related}, we present background on distributed submodular coordination and related work.  In \Cref{sec:problem}, we define the problem of \textit{Resource-Aware Distributed Submodular Optimization}.  Then, we present our algorithm in~\Cref{sec:algorithm}.  The theoretical analysis of the algorithm's approximation guarantees and decision time are presented in \Cref{sec:guarantee,sec:resource}, respectively.  \Cref{sec:experiments} presents the simulator and the evaluations.  All proofs are presented in the Appendix.

\myParagraph{Comparison with Preliminary Work~\cite{xu2022resource} and~\cite{xu2024performance}} This paper extends our preliminary work~\cite{xu2022resource} to include new theory, the AirSim extension, the evaluations, and proofs of all claims.  Particularly, this paper provides novel performance guarantees, including an a posteriori bound with network-design implications, and an extension of the a priori bound in~\cite{xu2022resource} to functions that are not submodular.  Also, this paper bounds the decision time of the algorithm, characterizing for the first time its communication and computation complexity as a function of the network size and the time to perform computations and r2r communications.  This paper introduces the AirSim-based simulator and the evaluations on the simulator. Instead,~\cite{xu2022resource} employed MATLAB simulations only without simulating r2r communication delays.  All proofs were omitted in~\cite{xu2022resource}, and here are presented for all original and new results.

We also compare this paper with our preliminary work~{\cite{xu2024performance}}. First of all,~{\cite{xu2024performance}} introduces a more complex problem formulation that requires the mesh-network topology to be co-optimized with the robots' actions.  The solution provided therein requires robots at fixed locations, instead of mobile sensors, which are the focus herein.
Finally, the results in~{\cite{xu2024performance}} hold in probability and are based on regret optimization that requires quadratic time to converge, whereas the results herein are deterministic and are based on discrete optimization that requires linear time to converge.

\section{Related Work}\label{sec:related}

{We discuss related work on (i) near-optimal but not necessarily real-time distributed submodular optimization, (ii) rapid but not necessarily near-optimal distributed multi-robot motion planning, (iii)  the trade-off of decision speed vs.~optimality, and (iv) high-fidelity communication simulators.}

\paragraph{{Near-optimal distributed submodular optimization}} The current approximation paradigms for submodular maximization problem~\cite{sviridenko2017optimal,fisher1978analysis} have been widely used in active information gathering with multiple robots~\cite{krause2008near,dames2017detecting,singh2009efficient,tokekar2014multi,atanasov2015decentralized,corah2019distributed,lauri2020multi,schlotfeldt2021resilient,biggs2022non,akcin2023fleet,cai2023energy} (see also the survey~\cite{sung2023survey}).  However, they may not scale.  The reason is that their communication and computation complexities are superlinear in the number of robots, particularly, quadratic or cubic (\Cref{tab:comparison}).  Instead, we provide an algorithm with linear communication and computation complexity. 

Quadratic or higher communication and/or computation complexity can become prohibitive in large-scale networks due to real-world communication and computation delays.  To illustrate the point, we give a toy example upon noting that communication delays are introduced by the limited r2r communication speeds.  For example, state-of-the-art r2r communication speeds range from less than $1$Mbps up to $100$Mbps~\cite{wu2021comprehensive}. Computational delays are caused by the time required to perform function evaluation.  This time depends on the processing required by the task at hand, \eg image segmentation for the said task of road coverage (Fig~\ref{fig:intro}).  Assume now that the total delay per communication and computation is on average $1$msec.  Then, for cubic decision-time complexity and $100$ robots ($|\calN|=100$), the total time delay is at the order $100^3 \cdot 1\text{msec}\simeq 17\text{min}$.  

The quadratic or higher complexity of~\cite{sviridenko2017optimal,fisher1978analysis} over r2r networks is due to their coordination protocols: they instruct the robots to retain and relay information about all or most other robots in the network.
Particularly,~\cite{sviridenko2017optimal} requires \textit{iterative decision-making via consensus} where at each iteration each robot needs to retain and transmit estimates of all robots' actions.~\cite{fisher1978analysis} requires \textit{sequential decision-making} where all currently finalized actions need to be relayed to all robots in the network that have not finalized actions yet.  

In more detail, the multi-robot coordination algorithm in~\cite{robey2021optimal}, inspired by~\cite{sviridenko2017optimal, fisher1978analysis}, achieves the best possible approximation bound for submodular maximization, namely, $1-1/e$.
However, it may require {tens} of minutes to terminate in simulated tasks of $10$ robots even with no simulated communication delays~\cite{xu2022resource}. 
The reason is that it requires a near-cubic number of iterations in the number of robots to converge (\Cref{tab:comparison}).
Similarly, for the \textit{Sequential Greedy} (\sg) algorithm~\cite{fisher1978analysis}, also known as \textit{Coordinate Descent}, which is the gold standard in robotics and control for submodular task optimization~\cite{krause2008near,dames2017detecting,singh2009efficient,corah2019distributed,tokekar2014multi,atanasov2015decentralized,lauri2020multi,schlotfeldt2021resilient,biggs2022non,akcin2023fleet,cai2023energy}, although it sacrifices some approximation performance to enable faster decision speed ---achieving the bound $1/2$ instead of the bound $1-1/e$--- has communication complexity that is cubic in the number of robots~\cite[Appendix~II]{xu2024performance}.  The reason is that it requires inter-robot messages that carry information about all the robots and a quadratic number of communication rounds over directed networks~\cite[Proposition~2]{konda2022execution}.  

\paragraph{Real-time multi-robot motion planning} {Current motion planning methods also employ approaches that do \underline{not} need to account for the submodular structure of the multi-robot task: the robots are expected to operate far from one another, collecting decoupled information. Then, the submodular objective simplifies to a distributed, modular objective: it can be expressed as the sum of decoupled functions, one for each robot, each computable by the robot based on its action only. 
Therefore, these methods often allocate robots to unexplored areas and instruct the robots to explore them independently.  Opportunistic communication is sufficient, when the robots happen to be nearby.}
For example, in collaborative exploration, the robots coordinate actions to move to the closest frontiers using shared map information in~\cite{yamauchi1999decentralized}. 
{\cite{zlot2002multi,berhault2003robot,sheng2006distributed,smith2018distributed} use an auction-based mechanism for task allocation among robots. The auction is also leveraged for high-probability communication maintenance in~\cite{sheng2006distributed}. In~\cite{yu2021smmr}, robots are assigned to different frontiers based on potential functions for distributed mapping and exploration. 
The robots set periodic meeting destinations in~\cite{gao2022meeting} to meet and coordinate actions during distributed exploration. \cite{klodt2015equitable} introduces a pair-wise coordination protocol where all pairs of neighboring robots sequentially coordinate trajectories to optimize their exploration task allocation. A similar pair-wise protocol is proposed in~\cite{zhou2023racer} where each robot coordinates trajectories with the neighbor that it has longest not communicated with.}

{
In contrast, this paper focuses on crowded scenarios where robots operate close to one another, and their objective is not merely exploration; it is discovery and/or monitoring of specific information, \eg landmarks, mobile targets, and areas or objects with specific semantics~\cite{atanasov2015decentralized,hughes2024multi,sung2023survey}.  Thus, to jointly maximize the submodular objective, the robots must coordinate actions~\cite{atanasov2015decentralized}: the submodular objective \underline{cannot} be decomposed as the sum of locally computable objectives, one for each robot, where the actions of each robot are unaffected by the actions of other robots.}
Additionally, {while the discussed task allocation works} are application-specific, \eg exploration~\cite{gao2022meeting,zhou2023racer}, our algorithmic results are general-purpose and apply to any distributed submodular optimization setting in robotics, control, machine learning, and beyond. 

\paragraph{Trade-off of decision speed vs.~optimality} 
{To enable rapid distributed submodular optimization}, we need to curtail the information explosion in robot mesh networks. Thereby, we need to limit what and how much information can travel across the network.  However, the consequence of imposing such information limitations is suboptimal coordinated actions. 

Current works have captured the suboptimality cost due to such limited information access for the case of the Sequential Greedy algorithm~\cite{gharesifard2017distributed,grimsman2019impact,biggs2022non}.  The provided characterizations are task-agnostic, holding true for the worst case over all possible submodular functions and action sets.
In contrast, we capture the suboptimality cost as a function of the task objective $f$ at hand, the robots' coordinated actions, and the current mesh network topology.   
Thus, our characterizations, being task specific, (i) can be tighter, and (ii) can be used by the robots to design communication neighborhoods that tune the trade-off of decision speed and near-optimality, subject to their communication bandwidth constraints.

\paragraph{Communication simulators}  
Papers focus on high-fi- delity simulations of wireless communications, including the simulation of protocols and communication limitations such as communication delays and package dropouts.  
\cite{lizzio2022implementation} simulates communication delays, and demonstrates its influence in consensus-based applications.~\cite{baidya2018flynetsim,calvo2021ros,acharya2023co} simulate communication channels and protocols, analyzing how distance and line-of-sight conditions impact communication delays. \cite{selden2021botnet} simulates realistic propagation models and scheduling functions for mobile radio-frequency communications.  We instead focus on enabling large-scale photo-realistic simulations with tens of robots, analyzing how communication delays impact  active information gathering with robot mesh networks.

\section{Resource-Aware Distributed Submodular Optimization
}\label{sec:problem}

We define the problem of \textit{Resource-Aware Distributed Submodular Optimization} (\Cref{pr:main}) with regard to the action coordination module in Fig.~\ref{fig:pipeline}. 
To this end, we use the notation:
\begin{itemize}[leftmargin=*]
    \item $\calN$ is the set of robots;
    \item  $\calE$ is the set of communication channels among the robots;
    \item $f(a\,|\,\calA)\triangleq f(\calA \cup \{a\})-f(\calA)$ is the marginal gain due to adding $a$ to $\calA$, given a set function $f:2^{\calV}\mapsto \mathbb{R}$.
\end{itemize}

{We also use the following framework. For easiness of illustration, we present the framework focusing on multi-robot information-gathering tasks such as simultaneous localization and mapping (SLAM), target tracking, surveillance~\cite{atanasov2015decentralized,corah2019distributed,schlotfeldt2021resilient}. 
\emph{Nonetheless, \Cref{pr:main} applies to any distributed optimization problem with a submodular objective function}.}

{\myParagraph{Robot Dynamics} 
We assume mobile robots that act as mobile sensors.  Their motion dynamics may take the form:
\begin{equation}\label{eq:dynamics}
    x_{i,t} = f_i(x_{i,t-1}, u_{i,t-1})+n_{i,t}, \quad t = 1, 2, \ldots, 
\end{equation}where $x_{i,t} \in \mathbb{R}^{n_{x_{i,t}}}$ denotes the state of robot $i$ at time step $t$,  $u_{i,t} \in \mathcal{U}_{i,t}$ denotes the control input applied to robot $i$, where $\mathcal{U}_{i,t}$ is a finite set of admissible control inputs~\cite{atanasov2015decentralized}, and $n_{i,t}$ is process noise, \eg Gaussian noise.
}

{\myParagraph{Target Dynamics} 
The robots use their sensors to measure and infer information about the environment. 
 For example, in SLAM and target tracking, at each time step $t$, the robots aim to estimate a target state vector $y_t$ that encapsulates the landmark locations or the mobile targets' to be localized~\cite{atanasov2015decentralized}. Generally, $y_t$ evolves with motion dynamics of the form:
\begin{equation}\label{eq:target}
    y_{t} = \phi(y_{t-1})+w_t, \quad t = 1, 2, \ldots, 
\end{equation}where $w_t$ denotes process noise, \eg Gaussian noise.
}

{\myParagraph{Sensor Model} 
For each robot $i$, we assume sensor model:
\begin{equation}\label{eq:sensor}
    z_{i,t} = g_i(x_{i,t}, y_t)+v_{i,t}(x_{i,t}), \quad t = 1, 2, \ldots, 
\end{equation}where $z_{i,t}$ denotes the measurement taken by robot $i$ at time $t$, and $v_{i,t}$ is noise that possibly depends on the robots' and targets' state.  For example, $v_{i,t}$ may be Gaussian noise with mean $\mu_{v_{i,t}}(x_{i,t}, y_t)$, and covariance $\Sigma_{v_{i,t}}(x_{i,t}, y_t)$. 
}

\myParagraph{Communication Neighborhood}  
{Before each action coordination step} (Fig.~\ref{fig:pipeline}), given the observed environment and states of the robots, the robots decide with which others to establish communication, subject to their onboard bandwidth constraints.  Specifically, we assume that each robot $i$ can receive information from up to $\alpha_i$ other robots due to onboard bandwidth constraints ($|\calN_i|\;\leq \alpha_i$). 

When a communication channel is established from  robot~$j$ to robot $i$, \ie $(j\rightarrow i) \in \calE$, then robot $i$ can receive, store, and process information from robot $j$.  The set of all robots that robot $i$ receives information from is denoted by $\calN_i$. We refer to $\calN_i$ as robot $i$'s \textit{neighborhood}.

\myParagraph{Communication Network} 
The resulting communication network can be directed and even disconnected.  When the network is fully connected (all robots receive information from all others), we call it \textit{fully centralized}. In contrast, when the network is fully disconnected (all robots receive no information from other robots), we call it \textit{fully decentralized}.

We assume communication to be synchronous.

\myParagraph{Communication Data Rate} All communication channels $(j\rightarrow i)\in\calE$ have finite \textit{data rates} (communication speeds). {In the simulations, we assume two values for the data rates: $0.25$ Mbps (simulating the Digi XBee 3 Zigbee 3.0 antenna), and $100.0$ Mbps (simulating the Silvus SL5200 antenna).}  Due to the finite data rates, the decision time of action coordination depends on both (i) the \textit{number of communication rounds} and (ii) the \textit{size of transmitted messages} it requires for the robots to find a joint plan. 
We assume that all communication channels have the same data rate in this paper.

\myParagraph{Objective Function} {At each action coordination step, the robots choose actions to maximize an objective function $f$ that captures the current multi-robot task. For example, in SLAM and target tracking, at each time step $t$, {$f$ can be~\cite{atanasov2015decentralized}:
\begin{equation}\label{eq:example_f}
    f(\{a_{i}\}_{i\myin\calN}) = \sum_{\tau=t+1}^{t+T} \mathbf{h}(y_\tau |\{z_{i,\;t+1:\tau}(a_i)\}_{i\myin\calN}),
\end{equation}
per the robot, sensing, and target models in \cref{eq:dynamics,eq:sensor,eq:target}, $$a_{i}\triangleq \{u_{i,t+1}, \dots, u_{i,t+T}\},$$  $z_{i,t+1:\tau}(a_{i})\triangleq \{z_{i,t+1}(a_{i}),\ldots,z_{i,\tau}(a_{i})\}$ denotes the measurements induced by $a_i$ up until $\tau$,}  $T$ is an action-planning look-ahead horizon, and  $\mathbf{h}(\cdot |\cdot)$ is an information metric such as the conditional entropy~\cite{atanasov2015decentralized} or the mutual information~\cite{corah2019distributed}.   When $\mathbf{h}$ is conditional entropy and the process and sensor noises in \cref{eq:dynamics,eq:sensor,eq:target} are uncorrelated Gaussian, then \cref{eq:example_f} is computable a priori given any action set $\{a_{i}\}_{i\myin\calN}$: $\mathbf{h}$ is then equal to the $\log\det$ volume of the Kalman filtering uncertainty over the horizon $T$ and is, thus, independent of the measurements' realization~\cite{atanasov2015decentralized}.
}
 
{Conditional entropy, mutual information, and, more broadly, {covering functions}~\cite{atanasov2015decentralized,corah2019distributed,corah2018distributed,robey2021optimal, downie2022submodular} are used to model active information-gathering 
tasks such as target tracking, SLAM, and surveillance~\cite{atanasov2015decentralized,corah2019distributed,schlotfeldt2021resilient}.} These functions capture how much {information} is observed given the actions of all robots.   
They belong to a broader class of functions that satisfy the properties below (\Cref{def:submodular,def:conditioning}).

\begin{definition}[Normalized and Non-Decreasing Submodular Set Function{~\cite{fisher1978analysis}}]\label{def:submodular}
A set function $f:2^\calV\mapsto \mathbb{R}$ is \emph{normalized and non-decreasing submodular} if and only if 
\begin{itemize}[leftmargin=*]
\item $f(\emptyset)=0$;
\item $f(\calA)\leq f(\calB)$, for any $\calA\subseteq \calB\subseteq \calV$;
\item $f(s\,|\,\calA)\geq f(s\,|\,{\mathcal{B}})$, for any $\calA\subseteq {\mathcal{B}}\subseteq\calV$ and $s\in \calV$.
\end{itemize}
\end{definition}

Normalization $(f(\emptyset)=0)$ holds without loss of generality.  In contrast, monotonicity and submodularity are intrinsic to the function.    
Intuitively, if $f(\calA)$ captures the area \emph{covered} by a set $\calA$ of activated cameras, then the more sensors are activated $(\calA\subseteq \calB)$, the more area is covered $(f(\calA)\leq f(\calB))$; this is the non-decreasing property.  Also, the marginal gain of covered area caused by activating a camera $s$ \emph{drops} $(f(s\,|\,\calA)\geq f(s\,|\,{\mathcal{B}}))$ when \emph{more} cameras are already activated $(\calA\subseteq \calB)$; this is the submodularity~property.

\begin{definition}[2nd-order Submodular Set Function{~\cite{crama1989characterization,foldes2005submodularity}}]\label{def:conditioning}
$f:2^\calV\mapsto \mathbb{R}$ is \emph{2nd-order submodular} if and only if 
\begin{equation}\label{eq:conditioning}
    f(s\,|\,\calC) - f(s\,|\,\calA\cup\calC) \geq f(s\,|\,\calB\cup\calC) - f(s\,|\,\calA\cup\calB\cup\calC),
\end{equation}
for any \emph{disjoint} $\calA, \calB, \calC\subseteq \calV$ ($\calA \cap \calB \cap \calC =\emptyset$) and  $s\in\calV$.
\end{definition}

The 2nd-order submodularity is another intrinsic property of the function.  
Intuitively, if $f(\calA)$ captures the area \emph{covered} by a set $\calA$ of cameras, then \emph{the marginal gain of marginal gains} drops when more cameras are already activated.

\begin{problem}[Resource-Aware Distributed Submodular Optimization]\label{pr:main}
{Consider a normalized and non-decreasing submodular function $f\colon 2^{\prod_{i \in \calN}\calV_i}\mapsto\mathbb{R}$ that captures the optimization task at hand at the current action coordination step}.  Each robot $i \in \calN$  selects an action $a_i\in\calV_i$, \emph{using only information from and about its neighbors $\calN_i$,}
such that the actions $\{a_i\}_{i \myin \calN}$ jointly solve the optimization problem:\footnote{We extend our framework to any normalized, non-decreasing and merely submodular or approximately submodular function $f$ in the Appendix.}
\begin{equation}\label{eq:problem}
\max_{a_i\myin\mathcal{V}_i, \, \forall\, i\myin \calN} \; f(\,\{a_i\}_{i\myin \calN}\,).
\end{equation}
\end{problem}

{\begin{remark}[Generality]
\Cref{pr:main} applies to any distributed optimization problem in control, robotics, machine learning, and beyond where the objective function is non-decreasing and submodular,
    including the information-gathering tasks captured via eqs.~\eqref{eq:dynamics} to \eqref{eq:example_f}.\footnote{{We refer the reader to the papers~\cite{atanasov2015decentralized,corah2019distributed} for implementation details on rapidly computing single-robot actions $a_{i}\triangleq \{u_{i,t+1}, \dots, u_{i,t+T}\}$ that optimize $f$ per eqs.~\eqref{eq:dynamics} to \eqref{eq:example_f}, given actions for the remaining robots.}}
\end{remark}}

\Cref{pr:main} is resource-aware in that it requires each robot to coordinate actions with its neighbors only, receiving information only about them instead of more robots in the network. The reason is to curtail the explosion of information passing across the network and, thus, to enable rapid coordination.  This is in contrast to standard distributed methods that allow information about the whole network to travel to all other robots via information passing (multi-hop communication)~\cite{konda2022execution,corah2018distributed,liu2021distributed}. Multi-hop communication does not reduce the amount of information flowing in the network compared to centralized methods, introducing impractical communication {delays, as we will demonstrate in the simulations.} 

\section{Resource-Aware distributed Greedy (\alg) Algorithm} \label{sec:algorithm}

{
\setlength{\textfloatsep}{3mm}
\begin{algorithm}[t]
	\caption{\mbox{Resource-Aware distributed Greedy (\alg).}
	}
	\begin{algorithmic}[1]
		\REQUIRE\!robot $i$'s actions $\mathcal{V}_i$; neighborhood $\calN_i$; \\
  non-decreasing set function $f:2^{\mathcal{V}_\calN} \mapsto \mathbb{R}$; 
		\ENSURE robot $i$'s action $\singlesol_i$.
		\medskip
		
		\STATE $\agentselected_i\leftarrow\emptyset$;~~~$\solutionselected_i\leftarrow\emptyset$;~~~$\singlesol_i\leftarrow\emptyset$;\label{line:initiliaze} {\color{gray}// \!$\calI_i$ is the robots\\ in $\calN_i$ that have selected their actions; $\calA_i$ stores $\calI_i$'s selected actions; $a_i^\alg$ is robot $i$'s selected action}
        \WHILE {$\singlesol_i = \emptyset$}
        \STATE $a_i\gets \arg\max_{a \myin \mathcal{V}_i}\,f(\solutionselected_i\cup \{a\})-f(\solutionselected_i)$;
        \STATE $g_i\gets f(\solutionselected_i\cup \{a_i\})-f(\solutionselected_i)$;
        \STATE \textbf{receive} $\{g_\rob\}_{\rob\myin{\mathcal{N}_i\setminus\agentselected_i}}$;
        \IF {$i=\arg\max_{\rob \myin \mathcal{N}_i\cup\{i\}\setminus\agentselected_i}\,g_{\rob}$}
        \STATE $\singlesol_i \leftarrow a_i$; {\color{gray}// $i$ has the best action candidate across $\mathcal{N}_i\cup\{i\}\setminus\agentselected_i$ and it selects this action candidate}
        \STATE \textbf{broadcast} $\singlesol_i$ \textbf{to} out-neighbors; {\color{gray}// out-neighbors can be different from $\calN_i$}
        \ELSE
        \STATE \textbf{denote} as $\calS^{\scenario{new}}_i\subseteq\calN_i\setminus \calI_i$ the set of neighbor(s) \\ that just selected action(s) in this iteration; {\color{gray}// $\calS^{\scenario{new}}_i$ may be empty as we explain in \Cref{subsec:algorithm}}
        \STATE \textbf{receive} $\singlesol_j$  
        \textbf{from} each robot $j \in \calS^{\scenario{new}}_i$;
        \STATE $\agentselected_i \leftarrow \agentselected_i \cup \calS^{\scenario{new}}_i$;
        \STATE $\solutionselected_i \leftarrow\solutionselected_i \cup \{\singlesol_j\}_{j\myin\calS_i^{\mathlarger{\scenario{new}}}}$;
        \ENDIF
		\ENDWHILE\label{line:end_for_3}
        \RETURN $\singlesol_i$.
	\end{algorithmic}\label{alg:dec_sub_max}
\end{algorithm}
}

\newcommand{\introFigTitleWidth}{0.2cm}
\newcommand{\introFigColWidth}{3.18cm}
\newcommand{\introFigSpacing}{\hspace{-2mm}}
\newcommand{\intoFigNameSpacing}{}
\newcommand{\advFigColWidth}{6cm}

\begin{figure*}[t!]
    \captionsetup{font=footnotesize}
	\centering
	\hspace{-9.7cm}
    \begin{minipage}{\columnwidth}
        \begin{tabular}{p{\introFigTitleWidth}|p{\introFigColWidth}p{\introFigColWidth}p{\introFigColWidth}p{\introFigColWidth}|p{\introFigColWidth}}%
        \hline
        \begin{minipage}{\introFigTitleWidth}%
        \end{minipage}
        &
        \multicolumn{4}{c|}{\sf \smaller\textbf{RAG}: Terminates in $2\,\tau_f\,|\calV_i|\,+\,\tau_c\,+\,\tau_\#$}
        &
        \begin{minipage}{\introFigColWidth}%
              \centering
              \rotatebox{0}{\sf \smaller\textbf{SG}: $5\,\tau_f\,|\calV_i|\,+\,10\,\tau_c$ \vspace{-4cm}} 
        \end{minipage}
        \\
        \hline 
        \begin{minipage}{\introFigTitleWidth}%
              \rotatebox{90}{\hspace{.3cm}\textbf{Line Graph}\vspace{-4cm}}
        \end{minipage}
        &
        \begin{minipage}{\introFigColWidth}%
            \centering%
            \includegraphics[width=\columnwidth]{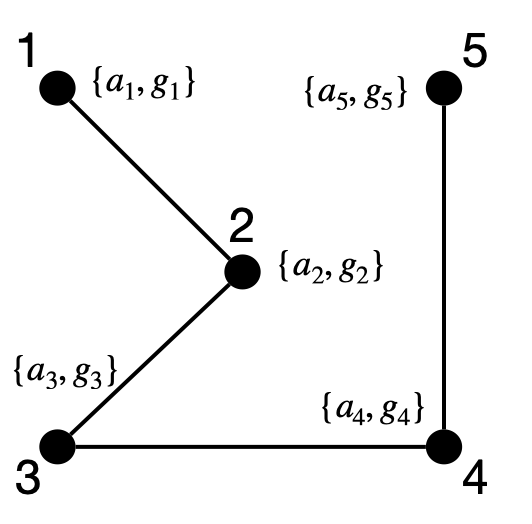} \\
            \caption*{(a) The 1st iteration of \alg starts. Each agent $i$ simultaneously finds its action candidate $a_i$ with the largest marginal gain $g_i$ from all available actions $\calV_i$. The operation takes $\tau_f\,|\calV_i|$.
            }
            \vspace{2mm}
        \end{minipage}
        &
        \begin{minipage}{\introFigColWidth}
              \centering%
              \includegraphics[width=\columnwidth]{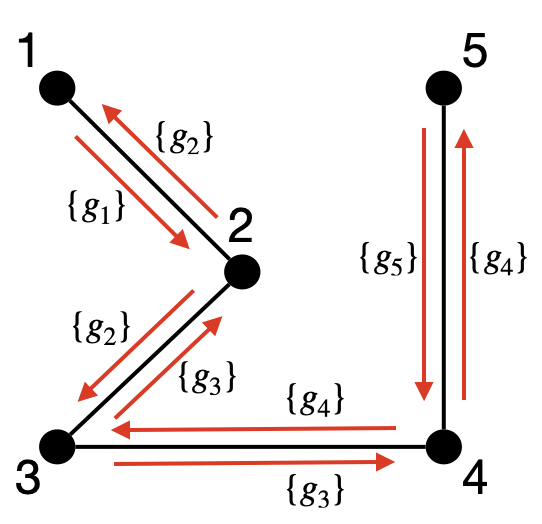} \\ 
              \caption*{(b) Each agent $i$ simultaneously receives $g_j$ from each neighbor $j\in\calN_i$. This takes $\tau_\#$.  Then, it compares $g_i$ with them. We assume $g_2=\max{(g_1, g_2, g_3)}$ and $g_4=\max{(g_3, g_4, g_5)}$. 
              }
              \vspace{2mm}
        \end{minipage}
        &
        \begin{minipage}{\introFigColWidth}
              \centering%
              \includegraphics[width=\columnwidth]{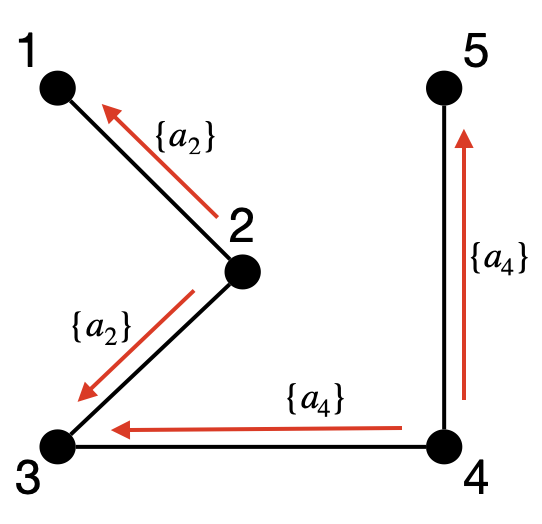} \\ 
              \caption*{(c) Thus, agents 2 and 4 get to select actions. Agents 1, 3, 5 simultaneously receive the actions selected by their neighbors.  This takes $\tau_c$. The 1st iteration ends. Only 1, 3, 5 continue.
              }
              \vspace{2mm}
        \end{minipage}
        &
        \begin{minipage}{\introFigColWidth}
              \centering%
              \includegraphics[width=\columnwidth]{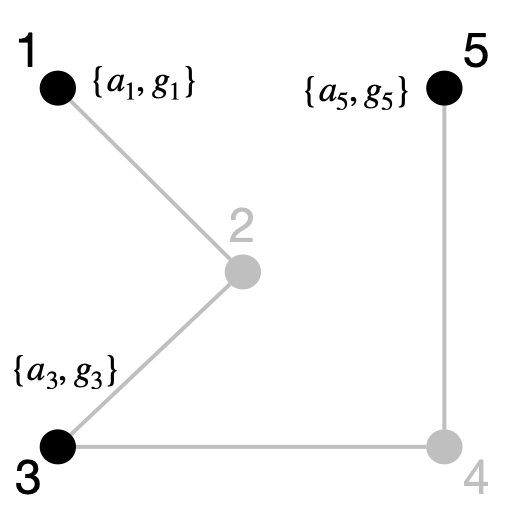} \\ 
              \vspace{-.5mm}
              \caption*{(d) The  2nd iteration starts. Given the received actions, agents 1, 3, 5 simultaneously select actions. This requires $\tau_f\,|\calV_i|$.  Then all agents have selected actions, and \alg terminates.
              }
              \vspace{2mm}
        \end{minipage}
        &
        \begin{minipage}{\introFigColWidth}%
              \centering%
              \includegraphics[width=\columnwidth]{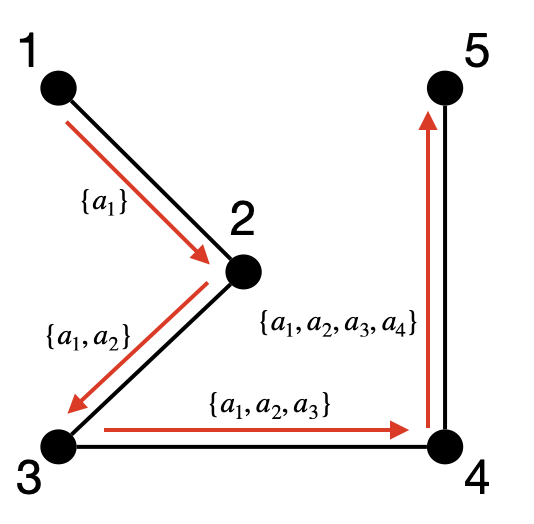} \\
               \vspace{-.35mm}
              \caption*{
          (e) Sequentially, from $i=1$ to $5$, agent $i$ receives $\{a_1,$ $\dots,a_{i-1}\}$,  in $\tau_c\,(i-1)$ time, then selects $a_i$, which takes $\tau_f\,|\calV_i|$ time, and then transmits $\{a_1,\dots,a_{i}\}$ to agent $i+1$.
              }
              \vspace{2mm}
        \end{minipage}
        \\\hline\hline
                \begin{minipage}{\introFigTitleWidth}%
        \end{minipage}
        &
        \multicolumn{4}{c|}{\sf \smaller\textbf{RAG}: Terminates in $2\,\tau_f\,|\calV_i|\,+\,\tau_c\,+\,\tau_\#$}
        &
        \begin{minipage}{\introFigColWidth}%
              \centering
              \rotatebox{0}{\sf \smaller\textbf{SG}: $5\,\tau_f\,|\calV_i|\,+\,17\,\tau_c$\vspace{-4cm}} 
        \end{minipage}
        \\
        \hline 
        \begin{minipage}{0.3cm}%
              \rotatebox{90}{\textbf{Star Graph}}
        \end{minipage}
        &
        \begin{minipage}{\introFigColWidth}%
              \centering%
              \includegraphics[width=\columnwidth]{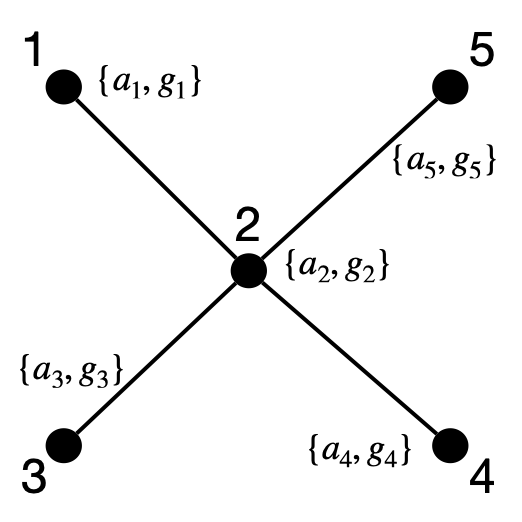} \\
              \caption*{(f) The 1st iteration of \alg starts. Each agent $i$ simultaneously finds its action candidate $a_i$ with the largest marginal gain $g_i$ from all available actions $\calV_i$. The operation takes $\tau_f\,|\calV_i|$.
              }
              \vspace{2mm}
        \end{minipage}
        &
        \begin{minipage}{\introFigColWidth}
              \centering%
        \includegraphics[width=\columnwidth]{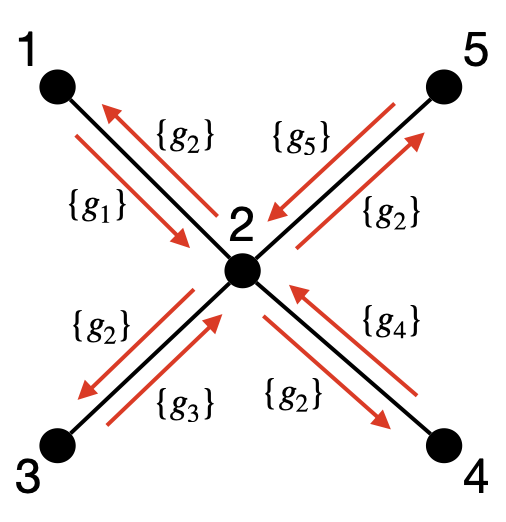} \\
        \vspace{-1mm}
        \caption*{(g) Each agent $i$ simultaneously receives $g_j$ from each neighbor $j\in\calN_i$. This takes $\tau_\#$.  Then, it compares $g_i$ with them. We assume $g_2=\max{(g_1, g_2, g_3, g_4, g_5)}$. 
        }
        \vspace{2mm}
        \end{minipage}
        &
        \begin{minipage}{\introFigColWidth}
              \centering%
        \includegraphics[width=\columnwidth]{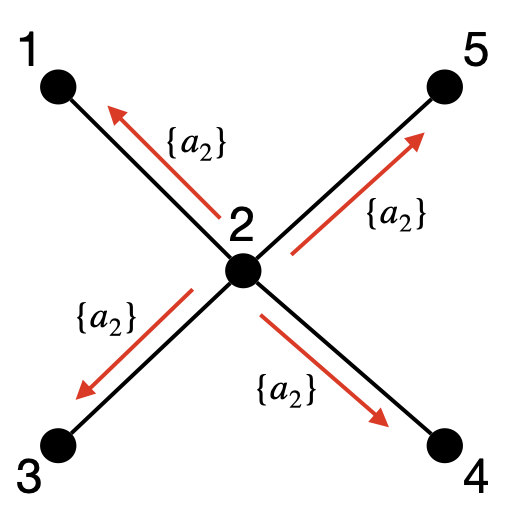} \\ 
         \vspace{-1mm}
        \caption*{(h) Thus, agent 2 gets to select an action. Agents 1, 3, 4, 5 simultaneously receive the action selected by their neighbor. This takes $\tau_c$. The 1st iteration ends. Only 1, 3, 4, 5 continue.
        }
        \vspace{2mm}
        \end{minipage}
        &
        \begin{minipage}{\introFigColWidth}
              \centering%
        \includegraphics[width=\columnwidth]{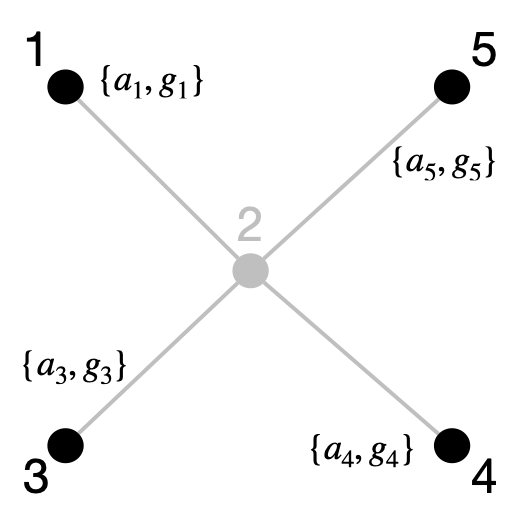} \\ 
        \caption*{(i) The 2nd iteration starts. Given the received action, agents 1, 3, 4, 5 simultaneously select actions. This requires $\tau_f\,|\calV_i|$. Then all agents have selected actions, and \alg terminates.
        }
         \vspace{2mm}
        \end{minipage}
        &
        \begin{minipage}{\introFigColWidth}%
        \centering%
        \includegraphics[width=\columnwidth]{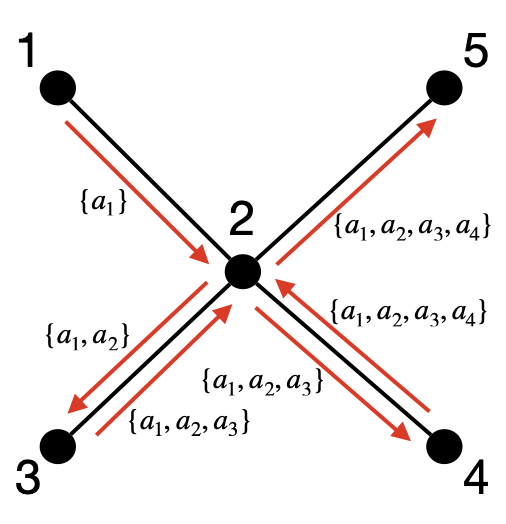} \\
         \vspace{-1.25mm}
        \caption*{(j) From $i=1$ to $5$, agent $i$ receives $\{a_1,\dots,a_{i-1}\}$, possibly via $r_i$ relay nodes (takes  $\tau_c\,(r_i+1)\,(i-1)$), then selects $a_i$ (takes $\tau_f\,|\calV_i|$), and then transmits $\{a_1,\dots,a_{i}\}$ to agent $i+1$.
        }
        \vspace{2mm}
        \end{minipage}\\
        \hline
        \end{tabular}
\end{minipage} 
	\caption{\textbf{\alg vs. Sequential Greedy (\sg)~\cite{fisher1978analysis}.} The two algorithms are compared in their execution steps and in the time they need to terminate. We present the case of five agents in two scenarios with the communication networks being (i) an undirected line graph (top row), and (ii) an undirected star graph (bottom row). Nodes $1$ to $5$ are the agents $1$ to $5$; black lines are undirected communication links; $\{a_i, g_i\}$ denotes agent $i$'s newly updated action candidate $a_i$ with marginal gain $g_i$; $\{a_i\}$ and $\{g_i\}$ alongside red arrows are the actions and marginal gain values being transmitted between two agents; the transparent nodes are the agents that have selected their actions and thus already ended running \alg; and the transparent edges are the communication channels that have become ``disappeared'' since at least one end of them has finished \alg. The implementation of \alg ((a)--(d) and (f)--(i)) results in shorter decision time than \sg ((e) and (j)) in both scenarios with the line and star graphs, respectively. 
	}\label{fig:RAG-vs-SG}
\end{figure*}

We present the \textit{Resource-Aware distributed Greedy} (\alg) algorithm. 
Examples of how the algorithm works are given in Fig.~\ref{fig:RAG-vs-SG}.  Therein, we also compare \alg to the Sequential Greedy algorithm (\sg)~\cite{fisher1978analysis}.  \sg is the ``gold standard'' in submodular maximization.  \sg is presented in (\Cref{subsec:RAG-vs-SG}).

\subsection{The {Resource-Aware distributed Greedy} (\hspace{-.5mm}\scenario{\textit{RAG}}) Algorithm}\label{subsec:algorithm}

The pseudo-code of \alg, as it is used onboard an robot~$i$, is presented in \Cref{alg:dec_sub_max}.  The purpose of each iteration of \alg, namely, of each ``while loop'' (lines 2--15), is to enable robot $i$ to decide whether to select an action over its neighbors at this iteration or to pass because a neighbor has an action with a higher marginal gain.  If passing, then the robot must wait for a future iteration to select an action.  In more detail, at each ``while loop'':
\begin{itemize}[leftmargin=*]
    \item robot $i$ finds an action $a_i$ with the highest marginal gain $g_i$ {given the actions selected by neighbors $\calI_i\subseteq\calN_i$} so far (lines 3--4).
    \item robot $i$ receives the respective highest marginal gain $g_j$ of all neighbors $j$ that have \underline{not} selected an action yet, namely, of all $j\in\calN_i\setminus\calI_i$ (line 5). 
    \item robot $i$ compares $g_i$ with all $g_j$'s (line 6).
    \item If $g_i > g_j, \forall j\in\mathcal{N}_i\setminus\calI_i$, then robot $i$ selects $a_i$, \ie $a_i^\alg\gets a_i$, broadcast $a_i^\alg$, and \alg terminates onboard robot $i$ (lines 6--8 and 2, respectively).
    \item Otherwise, robot $i$ passes (line 9), and receives the actions selected {at this iteration} by its neighbors with the highest marginal gain among their respective neighbors, if any (line 11) ---the set of these neighbors is denoted as  $\calS_i^{\scenario{new}}$ (line 10). Particularly, $\calS^{\scenario{new}}_i$ may be empty if no neighbor can select an action per their onboard iteration of \alg.
\end{itemize}

\begin{remark}[Directed, Possibly Disconnected Communication Topology]\label{rem:disconnected} \alg is valid for directed and even disconnected communication topologies. For example, \alg can be applied to a robot $i$ that is completely disconnected from the network.
\end{remark}

\subsection{Comparison to the Sequential Greedy algorithm (\hspace{-.5mm}\scenario{\textit{SG}})}\label{subsec:RAG-vs-SG}

\alg is compared with \sg~\cite{fisher1978analysis} in Fig.~\ref{fig:RAG-vs-SG}. We rigorously present \sg next, and provide a qualitative comparison with \alg. {The rigorous comparison of runtime and approximation performance is postponed to \Cref{sec:guarantee,sec:resource}, where, \eg we prove that \alg's runtime scales linearly with the number of the robots whereas \sg's scales cubically}.
 
\sg instructs the robots to sequentially select actions such that the $i$-th robot in the sequence selects
\begin{equation}\label{eq:sga}
  a_{i}^\sg \,\in\, \max_{a\myin\calV_i}\;\; f(\,a\;|\;\{a_{1}^\sg,\ldots,a_{i-1}^\sg\}\,);
\end{equation}
\ie $a_{i}^\sg$ maximizes the marginal gain over the actions that have been selected by the $i-1$ previous robots in the sequence.
In contrast, \alg enables the robots to select actions in parallel, and even if \underline{not} all their neighbors have selected an action.

The above action-selection features of \alg~---parallelization and action selection before all neighbors have chosen an action--- can further speed up the algorithm's termination.  Also, they enable \alg to work on arbitrary communication topologies. In contrast, \sg requires a line path connecting all robots in the action-selection sequence.  If such a path does \underline{not} exist (see the star graph example in Fig.~\ref{fig:RAG-vs-SG}), the $i$-th robot in the action-selection sequence cannot communicate directly with the $(i+1)$-th robot.  Then, \sg requires extra communication rounds for message relaying, further delaying its termination.  Specifically, {given the limited communication speed of robot-to-robot communication channels~\cite{oubbati2019routing}}, the termination of \sg is delayed due to both the increased number of communication rounds, and the communication delay incurred from relaying the actions of multiple robots ---the robots in the sequence \mbox{that have chosen an action so far--- across the network}. 

\section{Approximation Guarantees: \\Centralization vs.~Decentralization Perspective}\label{sec:guarantee}

\begin{figure*}[t]
    \captionsetup{font=footnotesize}
	\begin{center}
	\hspace{-9cm}\begin{minipage}{\columnwidth}
	\hspace{-2mm}		
	    \begin{minipage}{0.63\columnwidth}%
            \centering%
            \includegraphics[width=.8\columnwidth]{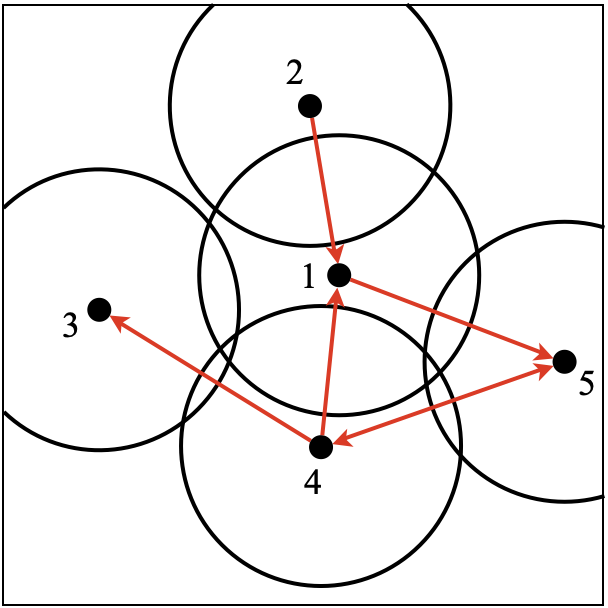} \\
            \caption*{(a) \textbf{Setup.} The robots are tasked to maximize the area covered among the available square area by picking locations to stay at. Each robot (dot) has a field of view (circle), and established communication channels with its neighbors  (red arrows). For example, for robot $1$, its neighbors are the robots in $\calN_1=\{2,4\}$, thus its non-neighbors are those in $\calN_1^c=\{3,5\}$. 
            }
		\end{minipage}~~~~~
		\begin{minipage}{0.63\columnwidth}
        \centering%
        \includegraphics[width=.8\columnwidth]{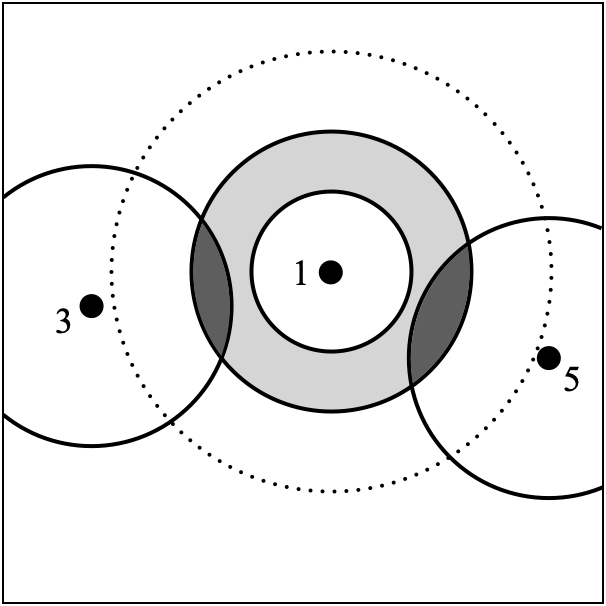} \\ 
        \caption*{(b) \textbf{Robot $1$'s non-neighbors $\calN_1^c$, $\scenario{coin}_{f,1}(\calN_1)$, and worst-case $\scenario{coin}_{f,1}$}. The non-neighbors $\calN_1^c=\{3,5\}$ are the robots that robot $1$ does not communicate range, \eg because of limited bandwidth or because they are outside its communication range (dashed circle). $\scenario{coin}_{f,1}(\calN_1)$ is depicted by the dark gray area, and its upper bound the light gray ring area  (\Cref{ex:coin-bound}). 
        }
		\end{minipage}~~~~~
		\begin{minipage}{0.63\columnwidth}%
			    \vspace{-16mm}
            \centering%
            \hspace*{-2mm}\includegraphics[width=.95\columnwidth]{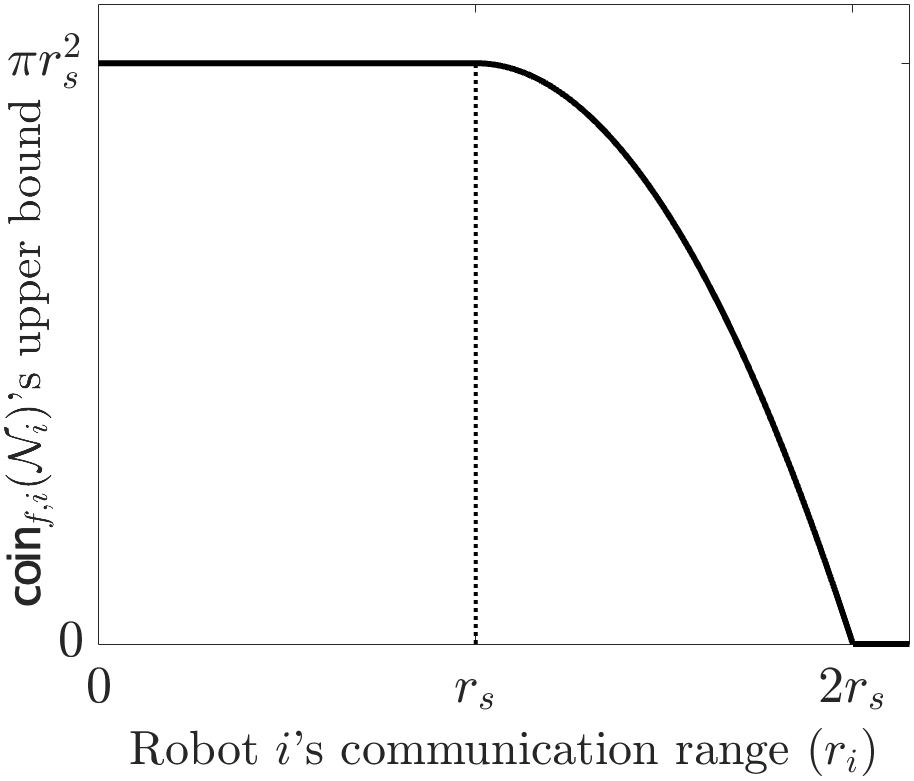} \\
             \vspace{-0.4mm}
            \caption*{(c) \textbf{Computable upper bound of $\scenario{coin}_{f,i}$ as a function of robot $i$'s communication range}.  $r_s$ is the robots' sensing radius (\Cref{ex:coin-bound}). 
            }
		\end{minipage}
	\end{minipage}
	\caption{\textbf{Multi-Robot Network for Area Coverage}. 
	(a) A scenario with $5$ robots, their communication network, and a limited square area that the robots are tasked to cover;  (b) Illustration of robot $1$'s non-neighbors $\calN_1^c$, $\scenario{coin}_{f,1}(\calN_1)$, and worst-case $\scenario{coin}_{f,1}$ $\scenario{coin}_{f,1}(\calN_1)$; (c)  Computable upper bound of $\scenario{coin}_{f,i}$ as a function of robot $i$'s communication range per the analysis in \Cref{ex:coin-bound}.
	}\label{fig:image_covering}
	\end{center}
\end{figure*} 

We present a priori and posteriori suboptimality bounds for \alg (\Cref{th:additive,th:posterior}).
To this end, we first introduce the notion of \emph{centralization of information} to quantify the a priori bound (\Cref{subsec:coin}). Then, we present the a priori bound (\Cref{subsec:approx-guarantee}) and the a posteriori bound (\Cref{subsec:posteriori}). Finally, we compare the a priori bound with the a priori bounds in the state of the art (\Cref{subsec:guarantees_comparison}).

\subsection{Centralization of Information}\label{subsec:coin}

We introduce the notion of \textit{\underline{c}entralization \underline{o}f \underline{in}formation} ($\scenario{coin}$). We use the notion to quantify the a priori suboptimality cost due to decentralization.  $\scenario{coin}$ measures how the agents' actions can overlap due to \underline{not} coordinating with their non-neighbors. We also relate $\scenario{coin}$ to curvature~\cite{conforti1984submodular} and pair-wise consistency~\cite{corah2018distributed} (\Cref{rem:curvature,rem:pairwise}), and show that $\scenario{coin}$ is a less conservative measure of action overlap.

We use the following notation and definition:
\begin{itemize}[leftmargin=*]\setlength\itemsep{0.15em}
    \item $\mathcal{N}_i^c\triangleq\ouragent\setminus\{\mathcal{N}_i\cup\{i\}\}$ is the set of robot $i$'s non-neighbors, \ie the robots beyond $i$'s neighborhood (see Fig.~\ref{fig:image_covering}(b)).
\end{itemize}

\begin{definition}[Curvature~\cite{conforti1984submodular}]
    Consider a function $f: 2^{\calV_\calN} \to \mathbb{R}$ that is non-decreasing and submodular. Without loss of generality, we assume that for any action $a \in \{\mathcal{V}_i\}_{i \in \calN}$, it holds that $f({a}) \neq 0$. The curvature of $f$ is defined as:
    \begin{equation}\label{eq:curvature}
    \kappa_f\triangleq 1-\min_{\calA \myin \calV_\calN}\;\min_{a\myin\calA}\;\frac{f(\calA)-f(\calA\setminus\{a\})}{f(a)}.
    \end{equation}
\end{definition}$\kappa_f$ measures how much an agent's action $a$ can overlap with other agents' actions. Particularly, $\kappa_f \in [0,1]$, and if $\kappa_f=0$, then  $f(\calA)-f(\calA\setminus\{a\})=f(a)$, for all $a\in\calA$, \ie \underline{no} agent's action overlaps with any other agent's actions. In~contrast, if $\kappa_f=1$, then there exists an action $a\in\calA$ such that $f(\calA)=f(\calA\setminus\{a\})$, \ie action $a$~has no contribution to $f(\calA)$ in the presence of all other agents. 

\begin{definition}[Centralization of Information]
\label{def:ourcurv}
Consider a function $f:2^{\calV_\calN}\mapsto$ $\mathbb{R}$ and a communication network $\{\calN_i\}_{i\myin\calN}$ where each agent $i\in \calN$ has an selected an action $a_i$. Then, agent $i$'s \emph{centralization of information} is defined as
\begin{equation}\label{eq:ourcurv}
  \ourcurv (\calN_i)\triangleq f(a_i) - f(a_i\,|\,\{a_j\}_{j\myin\calN_i^c}).
\end{equation}
\end{definition}

$\ourcurv$ measures how much $a_i$ can overlap with the actions of agent $i$'s non-neighbors.  In the best case where $a_i$ does \underline{not} overlap at all, \ie $f(a_i\,|\,\{a_j\}_{j\myin\calN_i^c})=f(a_i)$, then $\ourcurv=0$.  In the worst case instead where $a_i$ is fully overlapped, \ie $f(a_i\,|\,\{a_j\}_{j\myin\calN_i^c})=0$, then $\ourcurv= f(a_i)$. 

From an information-theoretic perspective, $\ourcurv$ measures how much the information collected by $a_i$ overlaps with the information collected by $\{a_j\}_{j\myin\calN_i^c}$.  Rigorously, if $f$ is an entropy metric, then  $\ourcurv$ is mutual information~\cite{cover2012elements}.  Thus, $\ourcurv=0$ if and only if the information collected by $a_i$ is decentralized from (independent of) the information collected by $\{a_j\}_{j\myin\calN_i^c}$.  In this sense, $\ourcurvnew$ captures the decentralization of information across the network. 

\begin{remark}[Relation to~Curvature{~\cite{conforti1984submodular}}]\label{rem:curvature}
$\ourcurvnew$ is a less conservative measure of action overlap compared to $\kappa_f$.  $\kappa_f$ measures the overlap of an agent's action with the actions of all other agents, whereas $\ourcurvnew$ measures the overlap of an agent's action with the actions of its non-neighbors only.  Particularly, we prove that, for all $i \in \calN$, $\ourcurv/f(a_i)\leq \kappa_f$ (see \Cref{prop:coin}, which is presented later on in this section). 
\end{remark}

\begin{remark}[Relation to~Pairwise Redundancy{~\cite{corah2018distributed}}]\label{rem:pairwise}
$\ourcurvnew$ generalizes the notion of \emph{pairwise redundancy} to capture the action overlap among multiple agents instead of a pair.
Specifically, given any two agents $i$ and $j$, their \emph{pair-wise consistency} is defined as $w_{ij}\triangleq \max_{s_i\myin\mathcal{V}_i}\, \max_{s_j\myin\mathcal{V}_j,\,j\myin\mathcal{N}_i}\,[f(s_i)-f(s_i\,|\,s_j)]$. 
In contrast, $\ourcurv$ captures the action overlap between an agent $i$ and its \emph{non}-neighbors, capturing that way the decentralization of information across the network.  
\end{remark}

By measuring how much agent $i$'s action overlaps with the actions of its non-neighbors, $\ourcurv$ equivalently captures agent $i$'s suboptimality cost due to not coordinating with its non-neighbors.  We thus expect that the more neighbors agent $i$ has the smaller is $\ourcurv$.  Indeed,
the following result holds:

\begin{proposition}[Monotonicity]\label{prop:coin}
For any $i\in\calN$, $\ourcurv(\calN_i)$ is non-increasing in $\calN_i$.  Its least and maximum values, attained for $\calN_i=\calN\setminus\{i\}$ and $\calN_i=\emptyset$, respectively, are as follows:
    \begin{equation}  \underbrace{0=\ourcurv(\calN\setminus\{i\})}_{\mathlarger{\begin{array}{c}
             \text{full centralization} 
             \end{array}}}\leq\ourcurv(\calN_i)\leq\underbrace{\ourcurv(\emptyset)=\kappa_f\,f(a_i)}_{\mathlarger{\begin{array}{c}
             \text{full decentralization} 
             \end{array}}}.
    \end{equation}
\end{proposition}

The sum of all $\{\ourcurv\}_{i\myin\calN}$ will be used in the next section to characterize the global suboptimality cost due to decentralization.  Given this characterization, we may want to enable the agents to pick their neighborhoods to minimize coin subject to their communication-bandwidth constraints.  $\ourcurvnew$ can be uncomputable a priori since agent $i$ will not have access to the actions of its non-neighbors.
Notwithstanding, finding a computable upper bound for $\ourcurv$ may be easy, as we demonstrate in the following example.

\begin{example}[Computable Upper Bound: Example of Area Coverage]\label{ex:coin-bound}
Consider an area coverage task where each robot carries a camera with a {circular field-of-view (FOV) of radius $r_s$} (Fig.~\ref{fig:image_covering}(a)). Consider that each robot $i$ has fixed its neighborhood $\calN_i$ by picking a communication range $r_i$.
Then, $\ourcurv$ is equal to the overlap of the FOVs of robot $i$ and its non-neighbors. 
Since the number of robot $i$'s non-neighbors may be unknown, 
an upper bound to $\ourcurv$ is the gray ring area in Fig.~\ref{fig:image_covering}(b), obtained assuming an infinite amount of non-neighbors around robot $i$, located just outside the boundary of $i$'s communication range. Specifically, 
\begin{equation}\label{eq:communication_range}
\ourcurv \leq \max(0, \;\pi [r_s^2-(r_i-r_s)^2]). 
\end{equation} 
The bound as a function of the communication range $r_i$ is plotted in Fig.~\ref{fig:image_covering}(c). It tends to zero for increasing $r_i$, as expected. When the distance of agent $i$ for its nearest non-neighbor is larger than $2r_s$, then the FOVs of agent $i$ and its non-neighbors cannot overlap, thus $\ourcurv=0$.  
\end{example}

\definecolor{OliveGreen}{rgb}{0,0.5,0}
\begin{table*}[t!]
    \renewcommand{\arraystretch}{1.9}
    \resizebox{\textwidth}{!}{
    \rotatebox{0}{
    \begin{minipage}{1.21\textwidth}
    \centering
    \begin{tabular}{c||c|c|c|c|c}
    \cline{2-6}
         & \multirow{2}{*}{\textbf{Method}} & 
         \multicolumn{3}{c|}{\textbf{Trade-off of Decision Time and Suboptimality Guarantee}} & \multirow{2}{*}{\parbox{3.1cm}{\centering\textbf{Communication Network Topology}}}\\
    \cline{3-5}
         & & \textbf{Decision Time: Computation} & \textbf{Decision Time: Communication} & \textbf{Suboptimality Guarantee} & \\
        
    \hline\hline
         \multirow{3}{*}{\rotatebox[origin=c]{90}{\textbf{Continuous}}} 
        & Robey et al. \cite{robey2021optimal} & $\Omega(\tau_f\,|\ouragent|^{2.5}\,/\,\epsilon)$ & $\Omega(\tau_\#\,|\mathcal{V}_\calN|\,|\ouragent|^{2.5}\,/\,\epsilon)$ & \blue{$(1-{1}/{e})\,\opt$ $-\,\epsilon$}  & connected, undirected \\  
    \cline{2-6}
         & Rezazadeh and Kia \cite{rezazadeh2023distributed} & $\Omega(\tau_f\,|\ouragent|^2\,\scenario{diam}(\calG)\,/\,\epsilon)$ & $\Omega(\tau_\#\,|\mathcal{V}_\calN|\,|\ouragent|^2\,\scenario{diam}(\calG)\,/\,\epsilon)$ & \blue{$(1-{1}/{e}-\epsilon)\,\opt$} & connected, undirected \\ 
    \cline{2-6}
         & Du et al. \cite{du2022jacobi} & $\sim\Theta(\tau_f\,|\ouragent|^2)$ & $\sim\Theta(\tau_c\,|\ouragent|^2)$ & $({1}/{2}-\epsilon)\,\opt$ & connected, undirected \\ 
    \hline\hline  
         & Liu et al. \cite{liu2021distributed} & $O(\tau_f\,|\mathcal{V}_i|\,|\ouragent|^2)$ & $O((\tau_c+\tau_\#)\,|\ouragent|^2\,\scenario{diam}(\mathcal{G}))$ & {${1}/{2}\,\opt$} & connected, directed \\
    \cline{2-6}
         \multirow{5}{*}{\rotatebox[origin=c]{90}{\textbf{Discrete}}} & \multirow{2}{*}{Konda et al. \cite{konda2022execution}} & \multirow{2}{*}{$\tau_f\,|\mathcal{V}_i|\,|\calN|$} & $O(\tau_c\,|\calN|^2)$ & \multirow{2}{*}{${1}/{2}\,\opt$} & connected, undirected \\
         \cline{4-4} \cline{6-6}
         \rule{0pt}{10pt} & & & $O(\tau_c\,|\calN|^3)$ & & strongly connected, directed \\
         \cline{2-6}  
         \rule{0pt}{20pt} & Corah and Michael \cite{corah2018distributed} & \parbox{3cm}{\centering $O(\tau_f\,|\mathcal{V}_i|\,/\,\epsilon)$, \\[2mm] $\leq\tau_f\,|\calV_i|\,|\calN|$} & \parbox{3cm}{\centering $O(\tau_c\,/\,\epsilon^2)$, \\[2mm] $\leq\frac{1}{2}\tau_c\,(|\calN|-1)^2$} & ${1}/{2}\,(\opt\,-\,\epsilon)$ & complete \\[3mm] 
    \cline{2-6}
        \rule{0pt}{15pt} & Gharesifard and Smith \cite{gharesifard2017distributed} &  {$\leq\tau_f\,|\mathcal{V}_i|\,|\calN|$} & $O(\tau_c\,|\calN|^2)$ & {$\eta\,\opt,\; \eta\in \left[\frac{1}{|\calN| - \omega(\calG_{\text{info}}) + 2}, \frac{\chi(\calG_{\text{info}})}{|\calN|}\right]$} & \multirow{2}{*}{\vspace{-3mm}\blue{possibly disconnected, directed}} \\[2mm] 
       \cline{2-5}
       \rule{0pt}{15pt} & Grimsman et al. \cite{grimsman2019impact} &  {$\leq\tau_f\,|\mathcal{V}_i|\,|\calN|$} & $O(\tau_c\,|\calN|^2)$ & {$\psi\,\opt,\; \psi\in \left[\frac{1}{\alpha^\star(\calG_{\text{info}})+1}, \frac{1}{\alpha^\star(\calG_{\text{info}})}\right]$} & \\[2mm] 
         \cline{1-6}  
        \rule{0pt}{22pt} & \alg (this paper) & $\blue{\leq\tau_f\,\max_{i\myin\calN}{(|\mathcal{V}_i|\,|\mathcal{N}_i|)}}$ & $\blue{\leq(\tau_c+\tau_\#)\,(|\calN|-1)}$ & \parbox{3.8cm}{\centering ${1}/{2}\left[\opt-{\sum_{i\in\calN} {\ourcurv(\calN_i)}}\right]$, \\[1.5mm] $\xi\,\opt,\; \xi\in \left[1-\kappa_f, \frac{1}{1+\kappa_f}\right]$}  & \blue{possibly disconnected, directed} \\[3.7mm]
    \cline{2-6}
    \hline
    \end{tabular}
    \renewcommand{\arraystretch}{1}
    \caption{\textbf{{\sf \smaller\textbf{RAG}} vs.~State-of-the-Art Distributed Submodular Maximization Algorithms.} The state of the art is divided into algorithms that optimize (i) in the continuous domain, employing a continuous representation of $f$~{\cite{calinescu2011maximizing}}, and (ii) in the discrete domain. For each algorithm, we present its decision time, split into computation and communication times {(see \Cref{sec:resource} for definitions of $\tau_f$, $\tau_c$, and $\tau_\#$)}, and suboptimality guarantee. We also specify the communication network topology that is required by each algorithm. The best performances for each metric are in \blue{blue}.  We assume for simplicity that $|\calV_i|\,=|\calV_j|, \forall i,j\in\calN$.
    }
    \label{tab:comparison}
        \end{minipage}}
        }
\end{table*}

\subsection{A Priori Suboptimality Bound of \hspace{-.5mm}\scenario{\textit{RAG}}}\label{subsec:approx-guarantee}

We present the a priori suboptimality bound of \alg. 
by bounding $\ourcurvnew$ with a computable bound as a function of the agents' neighborhoods, we enable the agents to optimize their neighborhoods to maximize the suboptimality bound of \alg subject to their communication-bandwidth constraints.  

We focus the presentation on non-decreasing and doubly submodular functions, for sake of simplicity.  In Appendix~I (\Cref{cor:non-submodular}), we generalize the results to functions that are non-decreasing and submodular or approximately submodular. 

We use the following notation:
\begin{itemize}[leftmargin=*]\setlength\itemsep{0.15em}
    \item $\calA^\opt \myin\arg \,\max_{a_i\myin\mathcal{V}_i, \, \forall\, i\myin \calN} \, f(\,\{a_i\}_{i\myin \calN}\,)$, \ie $\calA^\opt$ is an optimal solution to Problem~\ref{pr:main};
    \item $\oursol\triangleq\{\singlesol_i\}_{i\myin\ouragent}$ is \alg's output for robots $\calN$.
\end{itemize}

\begin{theorem}[A Priori Suboptimality Bound]\label{th:additive}
Given a communication topology $\{\calN_i\}_{i\myin\calN}$, \alg guarantees:
    {\begin{equation}\label{eq:thm-3}
        f(\,\oursol\,)\,\geq \, \frac{1}{1+\kappa_f}\left[f(\,\mathcal{A}^\opt\,) - \kappa_f\sum_{i\in \calN}\ourcurv(\,\calN_i\,)\right].
    \end{equation}}
\end{theorem}

\Cref{th:additive} captures the intuition that when the agents coordinate with fewer other agents, then the approximation performance will deteriorate.  This intuition is made rigorous by applying \Cref{prop:coin} to \cref{eq:thm-3} along the spectrum from fully centralized to fully decentralized networks:

\begin{itemize}[leftmargin=*]
\item  If $\calG$ is \textbf{fully centralized} (all agents communicate with all), then the approximation bound in \cref{eq:thm-3} becomes: 
    {\begin{equation}\label{eq:thm-1}
     f(\,\oursol\,)\,\geq \, \frac{1}{1+\kappa_f}\,f(\,\mathcal{A}^\opt\,),
    \end{equation}}\ie~\alg is near-optimal, matching the approximation ratio $1/(1+\kappa_f)$ of the seminal \sg algorithm~\cite{conforti1984submodular}.  The bound $1/(1+\kappa_f)$ is near-optimal since the best possible bound for the optimization problem in \eqref{eq:problem} is $1-\kappa_f/e$~\cite{sviridenko2017optimal}. 

\item  If $\calG$ is \textbf{in between} fully centralized and fully decentralized, then the approximation bound in \cref{eq:thm-3} captures as is the cost of decentralization.  It does so through $\ourcurvnew$, which measures how the agents' actions overlap due to not coordinating with all others.   Specifically, as the network becomes less and less centralized (the agents have less neighbors), then suboptimality bound in \cref{eq:thm-3} deteriorates since, for all $i \in \calN$, $\ourcurv(\calN_i)$ increases when the neighborhood $\calN_i$ becomes smaller (\Cref{prop:coin}).  

\item If $\calG$ is \textbf{fully decentralized} (all agents isolated), then the approximation bound in \cref{eq:thm-3} becomes: 	
\begin{align}\label{eq:thm-4}
        \hspace{-3.8mm}f(\,\oursol\,)&\,\geq \, \frac{1}{1+\kappa_f}\left[f(\,\mathcal{A}^\opt\,) - \kappa_f \sum_{i\in \calN}\ourcurv(\,\emptyset\,)\right]\\
         &\,\in \, \left[1-\kappa_f,\, \frac{1}{1+\kappa_f}\right]f(\,\mathcal{A}^\opt\,).\label{eq:thm-5}
\end{align}

\Cref{eq:thm-4} captures the intuition that when the agents' actions do not overlap, then no communication still leads to near-optimal performance.  For example, per the area coverage \Cref{ex:coin-bound}, when the agents are sufficiently far away such that their field of views cannot overlap upon executing their actions, then $\ourcurv(\,\emptyset\,)=0$ for all $i\in\calN$.  Particularly, then the bound in \cref{eq:thm-4} becomes $1/(1+\kappa_f)$, matching the fully centralized performance.

\quad In the worst case, the bound in \cref{eq:thm-3} takes the value $1-\kappa_f$, and becomes zero when the actions of all agents fully overlap with each other ($\kappa_f=1$).  This is inevitable since all agents ignore all others and thus cannot coordinate actions to reduce the overlap. 

\end{itemize}

The tightness of the bound will be analyzed in future work.  

\begin{corollary}[A Priori Bound with Approximate Greedy Selection]\label{cor:inaccuracy}
    If \Cref{alg:dec_sub_max}'s line 3 can only perform \emph{approximate} greedy selection such that $f(a_i)\geq \eta\, f(a_i^\opt\,|\,\calA_i), 0<\eta\leq 1$, then \alg guarantees:
    \begin{align}\label{eq:cor-1}
        f(\,\oursol\,)\,\geq \, &\frac{\eta}{1+\eta\kappa_f}\Bigg[f(\,\mathcal{A}^\opt\,) \nonumber\\
        &- \left(\frac{1}{\eta}-1+\kappa_f\right)\sum_{i\in \calN}\ourcurv(\,\calN_i\,)\Bigg].
    \end{align}Also, it holds that:
    \begin{align}\label{eq:cor-2}
        \frac{f(\,\oursol\,)}{f(\,\calA^\opt\,)} \geq
        \begin{cases} 
        \frac{\eta}{1+\eta\kappa_f}, & \calG \text{ is {fully centralized}}, \\
        \eta(1-\kappa_f), & \calG \text{ is {not fully centralized}}.
        \end{cases} 
    \end{align}
\end{corollary}

\alg still has near-optimal performance guarantees even with approximate greedy selection~\cite{singh2009efficient,corah2019distributed}. When $\eta=1$, \Cref{cor:inaccuracy} reduces to \Cref{th:additive}.

\subsection{A Posteriori Suboptimality Bound of \hspace{-.5mm}\scenario{\textit{RAG}}}\label{subsec:posteriori}

We present the a posteriori approximation bound of \alg (\Cref{th:posterior}).  
We recall the notation:
\begin{itemize}[leftmargin=*]
    \item $\calI_i$ is robot $i$'s neighbors that select actions prior to $i$ during the execution of \alg.
\end{itemize} 

\begin{theorem}[A Posteriori Suboptimality Bound]\label{th:posterior}
Given the actions $\{a_i^\alg\}_{i\myin\calN}$ selected by the agents, \alg guarantees:\footnote{\Cref{th:posterior} holds true for $f$ non-decreasing and submodular and not necessarily 2nd-order submodular.}
    \begin{equation}\label{eq:posterior}
        \hspace{3mm}f(\,\oursol\,)\, \geq\, f(\,\mathcal{A}^\opt\,)\, -\, \kappa_f\sum_{i\in \calN}f(\,a_i^\alg\,|\,\calA_{\calI_i}^\alg\,).
    \end{equation} 
\end{theorem}

\Cref{th:posterior} captures the suboptimality cost due to decentralization, similarly to \Cref{th:additive}.  In contrast to \Cref{th:additive}, \Cref{th:posterior} captures the decentralization cost as a function of the action overlap between agent $i$ and its neighbors, instead of its non-neighbors.  As such, \Cref{th:posterior} captures the intuition that the larger agent $i$'s neighborhood is, the better the suboptimality guarantee can be since then agent $i$ would have the chance to coordinate actions with more agents. 
 
\begin{proposition}[\mbox{Approximate Submodularity of A Posteriori} Bound]\label{th:posterior-submodular}
The right-hand side of \cref{eq:posterior} is non-decreasing and approximate submodular as a function of $\{\calI_i\}_{i\myin \calN}$.
\end{proposition}

The proposition implies that although the approximation performance will improve if the agents have more neighbors, the gained improvement diminishes.  

\begin{remark}[{Trade-Off between Decision Speed and Optimality}]\label{rem:larger-not-better} 
{For larger neighborhoods, the marginal increase in the approximation guarantee is negatively outweighed by a greater increase in decision time.
The reason is that the gain in the approximation guarantee diminishes as the neighborhoods become larger (\Cref{th:posterior-submodular}) while the decision time may increase linearly for \alg and superlinearly for the state of the art, as we present in \Cref{sec:resource}.}
\end{remark}

{The appropriate size of neighborhoods that balances the trade-off for the task at hand can be found via experimentation in high-fidelity simulations, as we illustrate in the experiments.} 

\subsection{Comparison to the State of the Art}\label{subsec:guarantees_comparison}

\newcommand{\figColWidth}{5.2cm}
\newcommand{\txtColWidth}{0.8cm}

\begin{figure}[t!]
    \captionsetup{font=footnotesize}
        \Large
	\centering
    \resizebox{\columnwidth}{!}{
    \begin{minipage}{\textwidth}
        \renewcommand{\arraystretch}{1.4}
        \begin{tabular}{p{\txtColWidth}|p{\figColWidth}|p{\figColWidth}|p{\figColWidth}}%
        \cline{2-4}
        &
        \begin{minipage}{\figColWidth}%
            \centering%
            \includegraphics[width=\columnwidth]{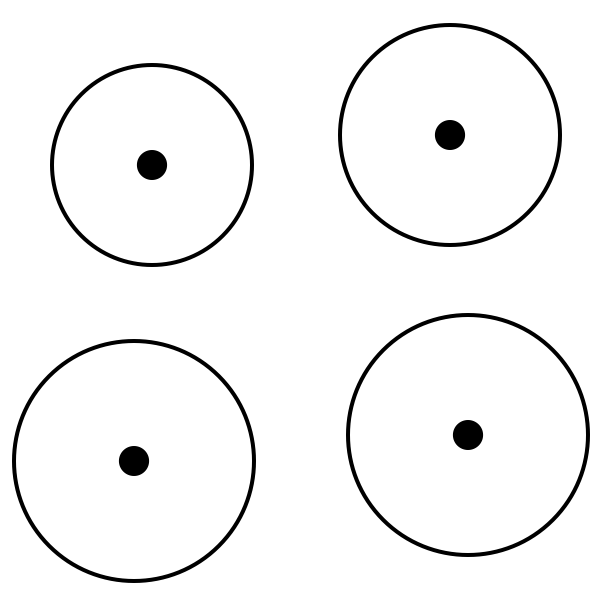} \\
            \caption*{\large\textbf{(a) fully decentralized}
            }
            \vspace{1mm}
        \end{minipage}
        &
        \begin{minipage}{\figColWidth}
              \centering%
              \includegraphics[width=\columnwidth]{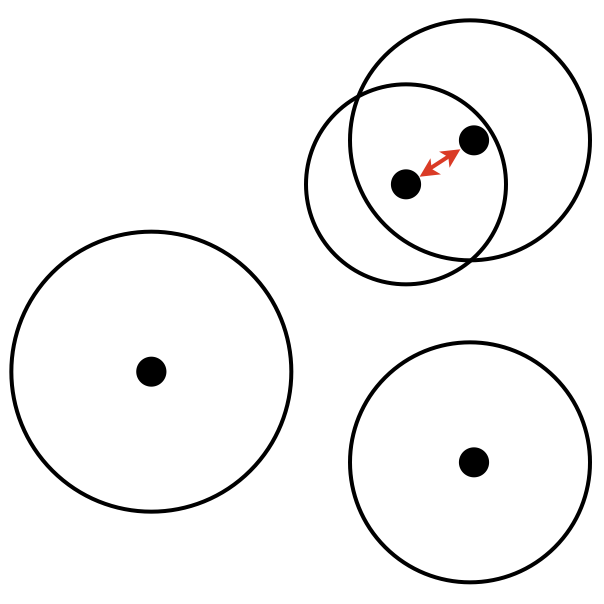} \\ 
              \caption*{\large\textbf{(b) partially decentralized}
              }
              \vspace{1mm}
        \end{minipage}
        &
        \begin{minipage}{\figColWidth}
              \centering%
              \includegraphics[width=\columnwidth]{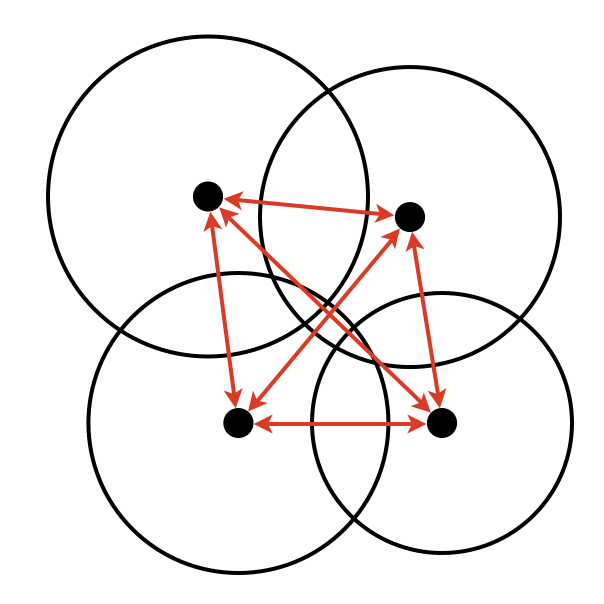} \\ 
              \caption*{\large\textbf{(c) fully centralized}
              }
              \vspace{1mm}
        \end{minipage} \\
        \hline
        \!~\cite{gharesifard2017distributed} & \centering $1/4$ &  \centering $[1/4, 1/2]$ & \hspace{2.1cm} $\color{blue}1/2$ \\
        \hline 
        \!~\cite{grimsman2019impact} & \centering $1/4$ &  \centering $[1/4, 1/3]$ & \hspace{2.1cm} $\color{blue}1/2$ \\
        \hline
        \alg & \centering $\color{blue}1$ & \centering $\color{blue}1/(1+\kappa_f)\geq 1/2$ & \hspace{7mm} $\color{blue}1/(1+\kappa_f)\geq 1/2$ \\
        \hline
        \end{tabular}
        \renewcommand{\arraystretch}{1}
    \end{minipage} 
    }
	\caption{\textbf{Comparison of Suboptimality Guarantees of~\cite{gharesifard2017distributed,grimsman2019impact} and \alg: Example of Area Coverage.} We make the comparison over an area coverage task with multiple drones, as in Fig.~\ref{fig:image_covering}(a), where the drones need to each select a short trajectory to maximize the total covered area at the next time step.  The stars are the drones' positions, the circles are their field-of-views (FOVs), and the red lines denote undirected communication links. The best performances for each metric are colored in \blue{blue}. (a) The drones are fully decentralized (no communication).  The bounds by~\cite{gharesifard2017distributed,grimsman2019impact} are both $1/4$.  For $|\calN|$ drones, instead of $4$, the bounds would be $O(1/|\calN|)$. Instead, \alg guarantees $1/(1+\kappa_f)=1$ since the drones are far enough for their FOV to overlap, $\ourcurv=0, \forall i$, and $\kappa_f=0$. 
 (b) The drones are partially decentralized such that only the physically close drones communicate. For this communication graph, that we assume to be the same as the information-action access graph, $\omega=2$, $\chi=2$, and $\alpha^\star=3$. Therefore, \cite{gharesifard2017distributed} gives $[1/4, 1/2]$ and \cite{grimsman2019impact} gives $[1/4, 1/3]$.  \alg still gives $1/(1+\kappa_f)\geq 1/2$ since non-neighbors are far away. 
 (c) In the fully centralized case,  
 \cite{gharesifard2017distributed,grimsman2019impact} give $1/2$ and \alg gives $1/(1+\kappa_f)$. 
 In all, \alg may provide tighter bounds than \cite{gharesifard2017distributed,grimsman2019impact} since it considers the actual function $f$ using $\ourcurvnew$ and $\kappa_f$, whereas \cite{gharesifard2017distributed,grimsman2019impact} are agnostic to $f$. 
	}\label{fig:RAG-vs-DSM}
\end{figure}

We summarize the approximation guarantees of the state of the art and \alg in \Cref{tab:comparison}.  
We observe the trade-off of decision time and optimality: the algorithms with the best suboptimality guarantees ---first six rows of the table, achieving the near-optimal $1/2$ or $1-1/e$~\cite{sviridenko2017optimal}--- can exhibit also the worst decision times, one to two orders {of the network size} higher than that of \alg.  
Among the remaining algorithms ---three last rows of the table--- \alg is the only algorithm that provides task-aware ($f$-based) performance guarantees and that quantifies the suboptimality guarantee as a function of each agent's local communication network.  
Instead, the guarantees of~\cite{gharesifard2017distributed,grimsman2019impact} are task-agnostic and can scale inversely proportional to the number of agents even when \alg can still guarantee the near-optimal $1/2$ (Fig.~\ref{fig:RAG-vs-DSM}). 

To discuss the suboptimality guarantees of~\cite{gharesifard2017distributed,grimsman2019impact} in more detail, we present their decision-making rule. Specifically,~\cite{gharesifard2017distributed,grimsman2019impact} use the following distributed submodular maximization (DSM) rule, introduced in~\cite{gharesifard2017distributed}:  
\begin{equation}\label{eq:dsm}
        a_{i}^\dsm \,\in\, \max_{a\myin\calV_i}\;\; f(\,a\;|\;\{a_{j}^\dsm\}_{j\myin\calN_i^{\text{in}}}\,),
\end{equation}
where $\calN_i^{\text{in}}\subseteq [i-1]$. \Cref{eq:dsm} generalizes \sg's rule in \cref{eq:sga} to the setting where agent $i$ has access only to the actions selected by the agents in $\calN_i^{\text{in}}\subseteq [i-1]$, instead of all agents that have selected an action before agent $i$.  
The information-access structure prescribed by the rule in~\cref{eq:dsm} can be represented as a directed acyclic graph $\calG_{\text{info}}$, where agent $i$'s neighbors in $\calG_{\text{info}}$ are the set $\calN_i^{\text{in}}$ of agents.\footnote{The graph $\calG_{\text{info}}$ is in general different from the communication graph $\calG$.  When an agent $i$ and an agent $j$ do not communicate (they are not neighbors in the communication graph $\calG$) but agent $j\in \calN_i^{\text{in}}$, then agent $j$'s action needs to be relayed to agent $i$ via other agents in $\calG$ that form a connected communication path in $\calG$ between agent $i$ and agent $j$.}  
Due to the limited information access,  the suboptimality guarantees of~\cite{gharesifard2017distributed,grimsman2019impact} take the form presented in \Cref{tab:comparison}, where $\omega(\calG_{\text{info}})$ is the clique number of $\calG_{\text{info}}$, $\chi(\calG_{\text{info}})$ is the chromatic number, and $\alpha^\star(\calG_{\text{info}})$ is the fractional independence number~\cite{godsil2001algebraic}. 

\section{Decision Time Analysis}\label{sec:resource}

We bound the time it takes for \alg to terminate. \alg's decision time scales linearly with the size of the network, up to two orders {of the network size} faster than the state of the art.   
We summarize the decision time of the state of the art and of \alg in \Cref{tab:comparison},
where we use the notation:
\begin{itemize}[leftmargin=3.5mm]
    \item $\tau_f$ is the time required for one evaluation of $f$;
    \item $\tau_c$ is the time for transmitting an action through a communication channel $(i\rightarrow j)\in\calE$;
    \item $\tau_\#$ is the time for transmitting a real number through a communication channel $(i\rightarrow j)\in\calE$; evidently, $\tau_\#\ll\tau_f$ and $\tau_\#\ll\tau_c$.
    \item $\scenario{diam}(\calG)$ is the diameter of a graph $\calG$, \ie the longest shortest path among any pair of nodes in $\calG$~{\cite{mesbahiBook}}.
\end{itemize}

We base our analysis on the observation that the decision time of any distributed  algorithm depends on the algorithm's: 
\begin{itemize}[leftmargin=3.5mm]
    \item \textit{computational complexity}, namely, the number of function evaluations required till termination (ignoring addition and multiplications as negligible in comparison); and
    \item \textit{communication complexity}, namely, the number of communication rounds needed till termination, accounting for the length of the communication messages per each round.
\end{itemize} 

\subsection{Decision Time of \hspace{-.5mm}\scenario{\textit{RAG}}}

We first analyze the computational and communication complexities of \alg and then present its decision time.  

\begin{proposition}[Computational Complexity]\label{prop:computation}
\alg requires each agent $i$ to perform at most $|\mathcal{V}_i||\mathcal{N}_i|$ function evaluations.
\end{proposition}
\begin{proof} For each agent $i$, $|\calA_i|$ increases by at least one with each ``while loop'' iteration of \alg. 
 At each such iteration,  agent $i$ needs to perform $|\calV_i|$ function evaluations to evaluate its marginal gain of all $v\in\mathcal{V}_i$ (lines 3--4). Since $|\calA_i|\,\leq|\calN_i|$, agent $i$ will perform at most $|\calV_i||\mathcal{N}_i|$ function evaluations.
\end{proof}

\begin{proposition}[Communication Complexity]\label{prop:communication}
\alg requires at most $|\calN|-1$ communication rounds where a real number is transmitted, and at most $|\calN|-1$ communication rounds where an action is transmitted. 
\end{proposition}
\begin{proof} The number of ``while loop'' iterations of \alg is at most $|\ouragent|-1$ because at each iteration at least one agent will select an action. Besides, each ``while loop'' iteration includes two communication rounds: one for transmitting a marginal gain value (line 5), and one for transmitting an action (lines 8 and 11). Hence, \Cref{prop:communication} holds. 
\end{proof}

\begin{theorem}[Decision Time of \alg]\label{th:speed}
\alg terminates in at most $(\tau_c+\tau_\#)\,(|\calN|-1)\,+\, \tau_f\,\max_{i\myin\calN}\,(|\mathcal{V}_i|\,|\mathcal{N}_i|)$ time.
\end{theorem}
\begin{proof}
\Cref{th:speed} holds from  \Cref{prop:computation,prop:communication}. 
\end{proof}

\subsection{Comparison to the State of the Art}
\label{subsec:RAG-SoA-time}

We summarize the decision times of the state of the art and  \alg in \Cref{tab:comparison}.
\alg scales linearly with the number of robots, $|\calN|$, whereas the state of the art scales at least quadratically with $|\calN|$.  
\alg has computational time that is linear in $|\calN_i|$, independent of $|\calN|$, and communication time linear in $|\calN|$.  

In \Cref{tab:comparison},
we assume for simplicity that $|\calV_i|\,=\!|\calV_j|, \forall i,j\in\calN$.  We divide the state of the art into algorithms that solve \Cref{pr:main} either in the continuous domain via employing the {multi-linear extension}~\cite{calinescu2011maximizing} of the set function $f$~{\cite{robey2021optimal,du2022jacobi,rezazadeh2023distributed}}, or in the discrete domain~\cite{corah2018distributed, konda2022execution,liu2021distributed,gharesifard2017distributed,grimsman2019impact}:\footnote{The continuous-domain algorithms employ consensus-based techniques {\cite{robey2021optimal,rezazadeh2023distributed}}, or algorithmic game theory~{\cite{du2022jacobi}}, and need to compute the multi-linear extension's gradients via sampling.}$^,$\footnote{The decision times of the continuous-domain algorithms depend on additional problem-dependent parameters ({such as Lipschitz constants, the diameter of the domain set of the multi-linear extension, and a bound on the gradient of the multi-linear extension}), which we make implicit in \Cref{tab:comparison} via the  $O$, $\Omega$, and $\Theta$ notations.}$^,$\footnote{The computational and communication times reported for \cite{du2022jacobi} are based on the numerical evaluations therein since a theoretical quantification is not included in \cite{du2022jacobi} and appears non-trivial to derive one as a function of $\calN$, $\epsilon$, or the other problem parameters.}  

\paragraph{Computation time}
\alg requires $\tau_f\,|\calN_i|\,|\calV_i|$ computation time.~The method in~\cite{konda2022execution} requires computation time $\tau_f\,\sum_{i\myin\calN}\,|\calV_i|\,=\!\tau_f\,|\calV_i|\,|\calN|$ since each agent $i$ 
needs to perform $|\calV_i|$~computations and the agents perform the computations sequentially. Given a pre-specified information access prescribed by directed acyclic graph (DAG) $\calG_{\text{info}}$, the methods in~\cite{gharesifard2017distributed,grimsman2019impact} also instruct the agents to select actions sequentially  leading to a computation time at most $\tau_f\,|\calV_i|\,|\calN|$. This time excludes the time needed to find the $\calG_{\text{info}}$ given an arbitrary communication graph $\calG$. {The method in~\cite{corah2018distributed} enables parallelized computation among agents by ignoring certain edges of an initially complete $\calG$, resulting in a computation time of $O(\tau_f\,|\calV_i|\,/\,\epsilon)\leq\tau_f\,|\calV_i|\,|\calN|$. Compared to~\cite{gharesifard2017distributed,grimsman2019impact}, the method in~\cite{corah2018distributed} also provides a distributed way to find which edges to ignore such that the suboptimality guarantee is optimized, a process that requires an additional computation time of $O(\tau_f\,|\mathcal{V}_i|^2\, (|\mathcal{N}|-1))$.} 
The remaining algorithms require longer computation times, proportional to $|\calN|^{2}$ or more. 

\paragraph{Communication time} 
\alg requires at most $(\tau_c+\tau_\#)\,(|\calN|-1)$ communication time. In~\cite{gharesifard2017distributed,grimsman2019impact}, the agents need to communicate over $\calG$ per the pre-specified information-access directed acyclic graph $\calG_{\text{info}}$ per the rule in \cref{eq:dsm}~\cite[Remark 3.2]{gharesifard2017distributed}. In Appendix~\red{IV}, we identify a worst case where this rule results in $O(\tau_c\,|\calN|^2)$ communication time.  This happens when each agent $i-1$ and agent $i$ in the decision sequence do not communicate directly, thus, the information of agent $i-1$ needs to be relayed to agent $i$ via other agents that form a connected communication path between the two.  
{The method in~\cite{konda2022execution}, which introduces a depth-first search (DFS) procedure to determine the best agents' ordering to run \sg~\cite{fisher1978analysis} over arbitrary (strongly) connected networks (instead of just line graphs), 
requires a worst-case communication time of $O(\tau_c\,|\calN|^3)$ for directed networks, and $O(\tau_c\,|\calN|^2)$ for undirected networks~\cite[Appendix II]{xu2024performance}}. For some $\epsilon$, the method in~\cite{corah2018distributed} may require less communication time for running the algorithm per se than other methods, but an additional communication time of $O(\tau_c\,|\calV_i|)$ is needed to distributively find a DAG that optimizes the algorithm's approximation performance.
The remaining algorithms require communication times proportional to $|\calN|^{2}$ or more.


\section{{Evaluation in Road Detection and Coverage}}\label{sec:experiments}

{We evaluate \alg in the provided simulator across four scenarios of road detection and coverage. The scenarios span two team sizes (15 and 45 robots), and two communication rates (0.25 Mbps), and 100 Mbps) ---up to three orders of magnitude faster than competitive near-optimal submodular optimization algorithms--- and, when the robots maintain a neighborhood size $\geq 2$, superior mean coverage.  The experiments also demonstrate \alg scales linearly with the network size, as predicted by our theoretical analysis: from the 15-robot to the 45-robot case,  \alg scales linearly, being at most 3x slower compared to the 15-robot case.  In contrast, the compared near-optimal algorithms scale cubically, being up to 60x slower compared to the 15-robot case.}

{Our code is available at {\url{https://github.com/UM-iRaL/Resource-Aware-Coordination-AirSim}}.}

\myParagraph{Common Simulation Setup across Simulated Scenarios} We first present the task of road detection (Fig.~\ref{fig:intro}), then the compared algorithms and the simulation pipeline (Fig.~\ref{fig:simulator}).

\paragraph{Road detection and coverage task}
Multiple aerial robots with onboard cameras are deployed in an unknown urban environment and tasked to {maximize the total new road area detected after each action coordination step} (Fig.~\ref{fig:intro}). {The environment is unknown to the robots (no map is available a priori), necessitating the robots to use their onboard cameras to perform exploration for road detection.

The robots are deployed close to each other relative to the size of their FOVs and, as a result, their FOVs may often overlap (Fig.~\ref{fig:intro}). \textit{Therefore, coordination becomes necessary for the robots to spread in the environment such that they detect different road segments and maximize the total new road area detected after each action coordination step.} 
}

To perform the task, given the currently visible environment, 
the robots coordinate how to move per the collaborative autonomy pipeline in Fig.~\ref{fig:pipeline}. Particularly, the task takes the form of the optimization problem in \Cref{pr:main} where $f$ denotes the number of road pixels captured by all robots' collective FOV after they traverse their agreed trajectories $\{a_i\}_{i\myin\calN}$ ---$f$ is non-decreasing, submodular, and 2nd-order submodular~\cite{corah2018distributed}--- and $\calV_i$ denotes robot $i$'s available trajectories at the current coordination round.  $\calV_i$ defines the possible directions that the robot can move in, and the speed the robot can move with.  For simplicity, we assume that every robot can move in any of the 8 cardinal directions ---N E S W NE SE NW SW--- relative to its body frame, for $10$ m at $3$ m/s.

Without loss of generality, the deployed robots are assumed to have the same onboard sensing and communication capabilities: All robots are equipped with (i) an inertial measurement unit (IMU); (ii) a GPS signal receiver; (iii) a downward-facing camera mounted on a gimbal that enables the camera to point to any of the 8 cardinal directions relative to the robot's body frame; and (iv) {a communication module ---either a Digi XBee 3 Zigbee 3.0 ($0.25$ Mbps) or Silvus Tech SL5200 ($100$ Mbps)--- for inter-robot communication}. Each robot can establish a few communication channels with robots within $100$ m range, subject to bandwidth constraints. 

\begin{figure*}[t]
    \captionsetup{font=footnotesize}
    \centering
    \includegraphics[width=.98\textwidth]{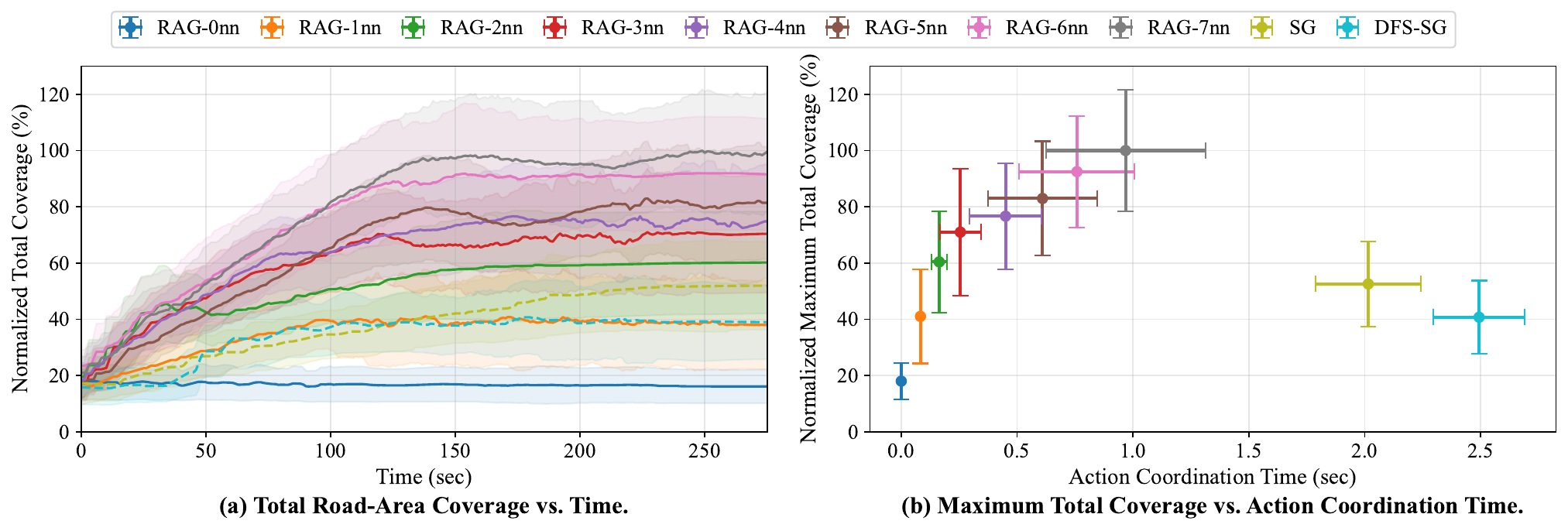}
    \caption{{\textbf{Comparative Analysis of Coverage Performance with 15 robots for $100$ Mbps data rate.}  
    The reported times include ROS1 delays of up to 0.18 sec per action coordination step. Particularly, in Figs.~\ref{fig:combined-coverage-analysis-15robots-100Mbps}--\ref{fig:combined-coverage-analysis-45robots-0p25Mbps}, 
    all coverage data are normalized by the mean results of the corresponding \alg-7nn instances, and the shared area and cross-hairs represent 1 standard deviation for y-axis in (a) and both axes in (b). }}
    \label{fig:combined-coverage-analysis-15robots-100Mbps}
\end{figure*}

\begin{figure*}[t]
    \captionsetup{font=footnotesize}
    \centering
    \includegraphics[width=\textwidth]{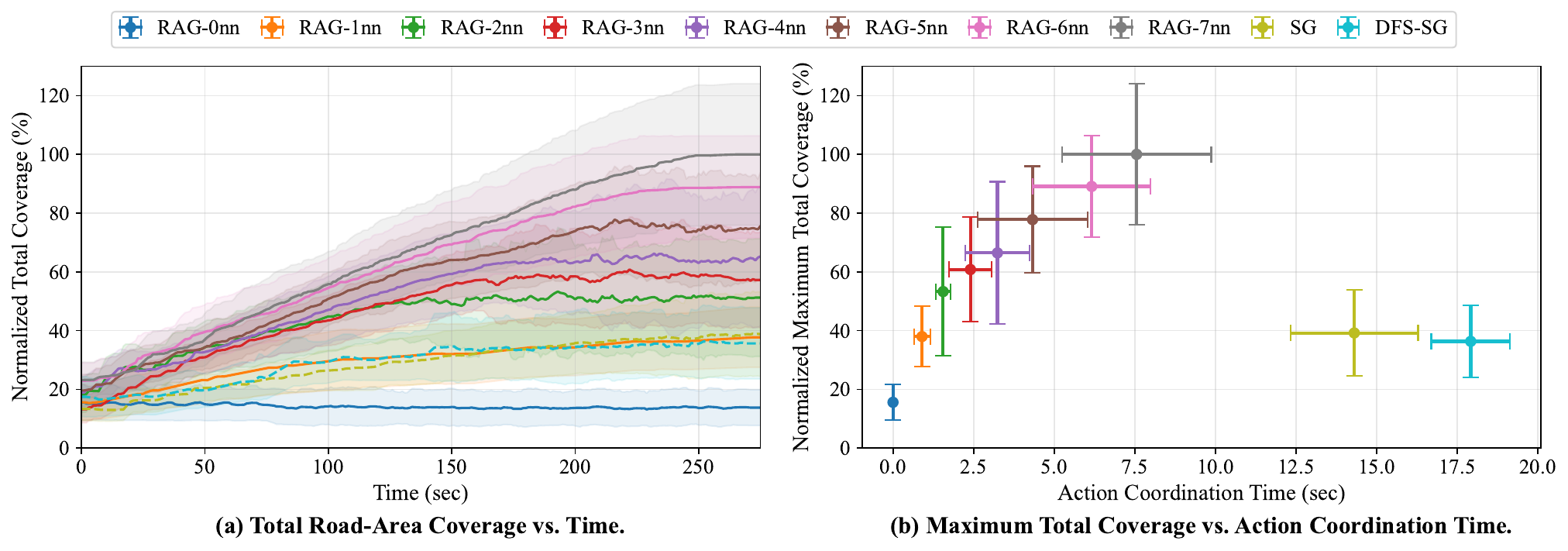}
    \caption{{\textbf{Comparative Analysis of Coverage Performance with 15 robots for $0.25$ Mbps data rate.} The reported times include ROS1 delays of up to 0.81 sec per action coordination step.}}
    \label{fig:combined-coverage-analysis-15robots-0p25Mbps}
\end{figure*}

\paragraph{Compared algorithms} Across various bandwidth constraints for the robots, we compare \alg with two competitive near-optimal algorithms: the Sequential Greedy (\sg) algorithm~\cite{fisher1978analysis}, also known as Coordinate Descent~\cite{atanasov2015decentralized}, and its state-of-the-art Depth-First-Search variant (\scenario{DFS-SG})~\cite{konda2022execution}.  {The algorithms are commonly used for information gathering with multiple robots (\eg see the papers~\cite{atanasov2015decentralized,corah2019distributed,schlotfeldt2021resilient} and the survey~\cite{sung2023survey}.} In more detail, the setup is as follows: 

We test \alg for different bandwidth constraints varying from $0$ up to {$7$}.\footnote{{Although neighborhood sizes larger than 7 are possible in the simulator, we limit them to at most size 7 to achieve a reasonable duration for completing each Monte Carlo trial. Particularly, the neighborhoods' scalability to $k>7$ is constrained by the msgpack-rpc protocol implementation over TCP/IP via rpclib. The current use of a single TCP port for multi-drone API calls creates a communication bottleneck, increasing client response times.  
For example, our experiments with 45 robots, which include 30 trials that simulate the robots' evolution for 500 sec for 9 different algorithms (\alg$k$-nn for $k=1,\ldots,7$, \sg, and {\sf DFS-SG}), require 3 days for each algorithm to be tested.
}}
In each case, the same bandwidth constraint applies to all robots.  Each such version of \alg is denoted by \alg-$k$nn, where {$k=0,\dots,7$}.
For each $k$, the communication network is {heuristically} determined by having each robot select $k$ nearest other robots as neighbors subject to the $100$ m communication range. If fewer than $k$ others are within the communication range, then all are selected as neighbors.  

The \sg algorithm requires the robots to be arranged on a {line graph} that defines the order in which the robots select actions per \cref{eq:sga} and enables the information relay from robots that have already selected actions to the robot currently selecting an action.  To ensure the existence of a line graph in our simulations of \sg, we randomly generate one, adjusting the robots' communication ranges to infinity.

The \scenario{DFS-SG} algorithm enables \sg to be applied to networks that are not necessarily a line graph, but the networks still need to be strongly connected. {To this end, we construct strongly connected graphs by first randomly constructing line graphs as for \sg, then randomly adding a few undirected edges to the line graphs, particularly, 30 edges for the 15-robot case and 90 edges for the 45-robot case.} At each {coordination step}, \scenario{DFS-SG} randomly picks the first robot to select an action, and the order of all other robots is determined by {a distributed method based on depth-first search.} {The resulting decision sequence may involve relay robots that transmit information between robots that are not directly connected and, thus, \scenario{DFS-SG} generally requires longer decision times than \sg.}

\begin{figure*}[t]
    \captionsetup{font=footnotesize}
    \centering
    \includegraphics[width=.98\textwidth]{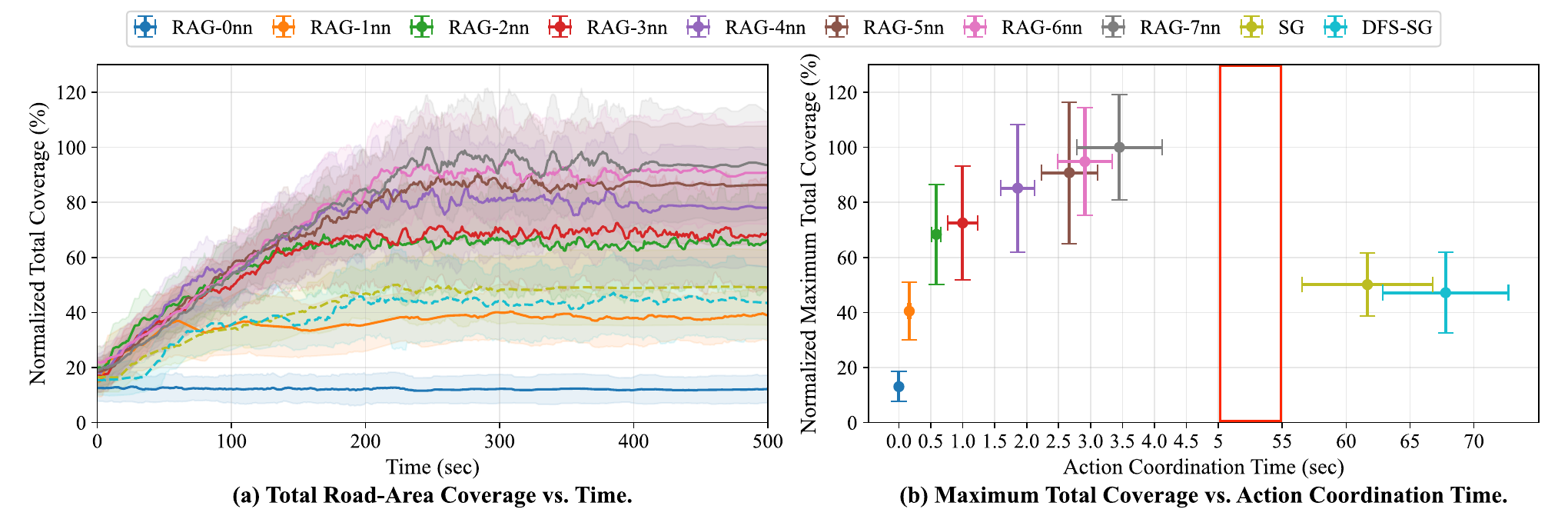}
    \caption{{\textbf{Comparative Analysis of Coverage Performance with 45 robots for $100$ Mbps data rate.} The reported times include ROS1 delays of up to 3.78 sec per action coordination step.}}
    \label{fig:combined-coverage-analysis-45robots-100Mbps}
\end{figure*}

\begin{figure*}[t]
    \captionsetup{font=footnotesize}
    \centering
    \includegraphics[width=\textwidth]{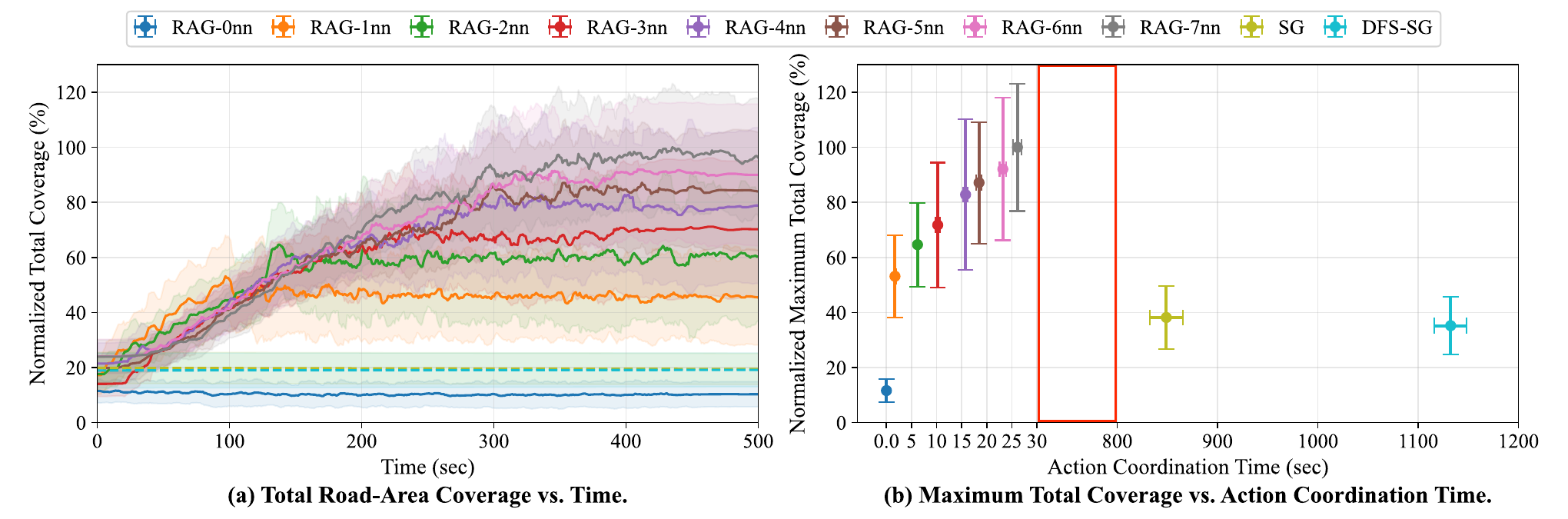}
    \caption{{\textbf{Comparative Analysis of Coverage Performance with 45 robots for $0.25$ Mbps data rate.} The reported times include ROS1 delays of up to 6.54 sec per action coordination step.}}
    \label{fig:combined-coverage-analysis-45robots-0p25Mbps}
\end{figure*}

\paragraph{Simulation pipeline}
The simulation pipeline consists of the following modules (Fig.~\ref{fig:simulator}):

\begin{itemize}[leftmargin=*]
    \item \textit{Network Self-Configuration}: This module applies only to \alg since only \alg enables the network to self-configure itself across the planning steps subject to the robots' bandwidth constraints and the robots' relative locations.  Particularly, as described also in the above paragraph for \alg-$k$nn, at the beginning of each planning step, each robot selects the $k$ nearest other robots within its communication range.  This neighborhood selection scheme is justified by \Cref{ex:coin-bound}.
    \item \textit{Perception}:
    The process is different for each algorithm since each algorithm requires the robots to process information received from different sets of robots:
    
    \quad For \alg, the robots need the following three operations to detect the nearby environment and evaluate where would be best to move: (i) At the beginning of each planning step, each robot hovers at a constant height (30 m) and rotates its camera to take a picture in each of the 8 cardinal directions relative to its body frame.  The captured FOV in each direction is of size $34.6\times 26.0$ m$^2$. Then, the robot uses semantic segmentation\footnote{{We simulate semantic segmentation using the in-built ground truth result of AirSim. The process runs at 20Hz, ensuring segmentation does not create a bottleneck for the replanning frequency of \alg.}} to detect the road segments in each of the 8 captured images. {The resulting segmented image has size 25 KB with label information after compression.} (ii) Each robot, upon receiving segmented images and relative poses from neighbors that have committed actions in specific directions, stitches these images together. That way, the robot reconstructs the collective FOV of its neighbors that have committed an action.  This reconstruction will be used next by the robot to select in which direction to move to cover the most new road area. To this end, (iii) the robot stitches its 8 captured images respectively with the previous reconstruction and counts the extra road pixels that can be covered in each of the corresponding 8 cardinal directions. 
    
    \quad For \sg and \scenario{DFS-SG}, the operations above are similar, with the modification being that instead of just the neighbors' segmented and stitched images, each robot leverages all previous robots' stitched images in (ii) and (iii). 

    \item \textit{Action Coordination}: 
    The process differs across algorithms:
    
    \quad For \alg, given their neighbors' already committed actions, the robots that are still deciding will first each pick a trajectory whose FOV gives the most amount of newly covered road area. Then, they will compare this amount with one another. Those who win their neighbors will commit to their picked trajectories and share the corresponding FOV and camera pose with neighbors. Otherwise, they will receive the FOVs and camera poses selected by newly committed robots and repeat the process above. 
    
    \quad For \sg, the robots make decisions sequentially in a line graph. Each robot, upon receiving the stitched FOVs of the trajectories selected by all predecessors, will first choose the trajectory whose FOV offers the most newly covered area based on previous robots' selections. It will then align and stitch this FOV with the received ones, and finally send the newly stitched FOVs to the next robot.
    
    \quad The coordination process for \scenario{DFS-SG} is similar to \sg, with the difference being that \scenario{DFS-SG} operates over a strongly connected network, given the robots' locations, bandwidth, and communication range, instead of a line graph. Thus, during \scenario{DFS-SG}, the $i$-th robot in the decision sequence may need to transmit FOVs to the $(i+1)$-th robot in the decision sequence via relay robots. 
    
    \item \textit{Control}: The robots simultaneously traverse the trajectories after coordination using ``simple\_flight'', a flight controller in AirSim. ``simple\_flight'' uses cascade PID controllers to drive the aerial robots to move in the selected direction with speed $3.0$ m/s for $10.0$ m. 
\end{itemize}

\subsection{Evaluations with 15 Robots}\label{subsec:15-robots}

We present the results for the 15-robot scenario. {Each compared algorithm was tested in 50 trials, each lasting 300 seconds.} The simulation setup is as follows.

\myParagraph{Simulation Setup}  Across the Monte Carlo tests, the robots are initialized near one another such that their FOVs largely overlap (\eg see the leftmost group of 15 robots in Fig.~\ref{fig:intro}). {{Therefore, coordination becomes necessary for the robots to spread in the environment such that they detect different road segments and, thus, maximize the total new road area detected after each action coordination step.}}

We ran the simulations on a 32-core CPU with 2 Nvidia RTX 4090 24GB GPUs and 128GB RAM on Ubuntu 18.04.

\myParagraph{Results} The results are summarized in {Figs.~\ref{fig:combined-coverage-analysis-15robots-100Mbps} (100 Mbps) and~\ref{fig:combined-coverage-analysis-15robots-0p25Mbps} (0.25 Mbps).}
{\alg-$k$nn achieves real-time planning for all $k$, and superior mean coverage for $k\geq 2$. 
} 

\paragraph{Road coverage over time} 
Figs.{~\ref{fig:combined-coverage-analysis-15robots-100Mbps}(a) and~\ref{fig:combined-coverage-analysis-15robots-0p25Mbps}(a)} present the total area of the road covered over time for each algorithm.
{\alg-$k$nn rapidly achieves superior  mean performance for $k\geq2$, and maintains comparable performance for $k=1$. The lower performance for $k=0$ is expected, since then no robot ever coordinates with any other. }

\paragraph{Trade-off between decision time vs.~road coverage} Figs.~\ref{fig:combined-coverage-analysis-15robots-100Mbps}(b) and~\ref{fig:combined-coverage-analysis-15robots-0p25Mbps}(b) {present the average action coordination time vs.~the peak coverage of the road area over the duration of the task (the peak values in Figs.~\ref{fig:combined-coverage-analysis-15robots-100Mbps}(a) and~\ref{fig:combined-coverage-analysis-15robots-0p25Mbps}(a)). 
For larger $k$, the marginal increase in
total coverage is negatively outweighed by a greater increase in decision time. The observation demonstrates the trade-off between decision speed vs.~optimality predicted in \Cref{sec:guarantee} (\Cref{rem:larger-not-better}).}

\subsection{Evaluations with 45 Drones}\label{subsec:45-robots}

We demonstrate the scalability of our algorithm from 15 to 45 robots. {Each compared algorithm was tested in 30 trials, each lasting 500 seconds. }
{The experiments demonstrate that \alg scales linearly, being at most 3 times slower compared to the 15-robot case.  In contrast, the compared near-optimal algorithms scale cubically, as predicted by our theoretical analysis, being $>27=3^3$ times slower compared to the 15-robot case: \eg~\sg's decision time increases from around 2 sec to 60 sec (100 Mbps case), a 30x increase, and from around 14 sec to 850 sec (0.25 Mbps case), a 60x increase.}

\myParagraph{Simulation Setup} Across the Monte Carlo tests, the robots are divided into three groups, with the robots within each group being deployed near to one another such that their FOVs largely overlap (Fig.~\ref{fig:intro}).  In such a setting, \alg automatically enables parallelized decision-making, achieving scalability and similar coverage as for the 15-robot setting, in contrast to the state-of-the-art near-optimal algorithms.

We ran the simulations on a remote server with 80 CPU cores (4x 2.4 GHz Intel Xeon Gold 6148), 360GB RAM, and 4 NVIDIA Tesla V100 16GB GPU. 

\myParagraph{Results}  The results are summarized in {Figs.~\ref{fig:combined-coverage-analysis-45robots-100Mbps} (100 Mbps) and~\ref{fig:combined-coverage-analysis-45robots-0p25Mbps} (0.25 Mbps)}. {\alg maintains superior mean coverage, with all qualitative observations from the 15-robot case applying here.  Per Figs.~\ref{fig:combined-coverage-analysis-45robots-100Mbps}(b) and~\ref{fig:combined-coverage-analysis-45robots-0p25Mbps}(b), \alg's coordination times scaled linearly compared to the 15-robot case.  In contrast, \sg's and \scenario{DFS-SG}'s coordination times scaled super-cubically, requiring now minutes per action coordination step at 100 Mbps and tens of minutes at 0.25 Mbps.  The results demonstrate the linear scalability of \alg and cubic scalability of the state of the art, as predicted by our theoretical analysis (\Cref{tab:comparison}).}

\section{Summary and Concluding Remarks} \label{sec:con}
{We introduced a distributed submodular optimization paradigm, Resource-Aware distributed Greedy (\alg), that enables scalable and near-optimal coordination over robot mesh networks.  The framework applies to any distributed submodular optimization task.  In this paper, we applied it to active information-gathering,  demonstrating \alg's performance in simulated scenarios of road detection with up to 45 robots. 
In the simulations, \alg enabled real-time planning, up to 3 orders of magnitude faster than competitive near-optimal algorithms, while also achieving superior mean coverage performance.
To enable the simulations, we extended the high-fidelity and photo-realistic AirSim by enabling a scalable collaborative autonomy pipeline to tens of robots while simulating realistic r2r communication messages and speeds.}
 
{In our future work, we will apply \alg to distributed metric-semantic SLAM where the r2r communication messages scale to MBs~\cite{tian2022kimera,liu2024slideslam}, 40 times larger than the messages of 25 KB we considered for the road detection task in this paper.}

We also plan the following algorithmic contributions.  
(i) \alg assumes synchronous communication.  Although \alg can be modified to handle such cases, \eg by instructing each robot to execute its action without first waiting for all other robots to select actions, its near-optimality guarantees provided in this paper may no longer hold. Our future work will extend our theoretical and algorithmic analysis beyond the above limitations. 
(ii) We will also extend our results to handle effective task execution over long time horizons.  For example, in collaborative mapping over long time horizons, the team needs to stay updated on the areas that have been mapped, such that the current plans do not repeat past actions. This is in addition to the focus of this paper that only the current plans among the robots do not overlap. To this end, we will handle network connectivity constraints. (iii) Finally, we will enhance our simulator by integrating the simulation of realistic communication channels and protocols towards communication-aware and -efficient coordination algorithms.


\appendices
\section*{Appendix I}\label{app:th1}
\paragraph*{Proof of \Cref{prop:coin}} Consider robot $i\in\calN$ and two disjoint robot sets $\calB_1, \calB_2\subseteq\calN\setminus\{i\}$, we have 
\begin{align}
    &\ourcurv(\calB_1\cup\calB_2) - \ourcurv(\calB_1) \nonumber\\
    &= f(a_i\,|\,\{a_j\}_{j\myin\calB_1\cup\calB_2}) - f(a_i\,|\,\{a_j\}_{j\myin\calB_1}) \leq 0,
\end{align}where the inequality holds since $f$ is submodular. Hence, $\ourcurv(\calN_i)$ is non-increasing in $\calN_i$. Therefore, for any $i\in\calN$, $\ourcurv(\calN_i)$ achieves the lower bound $\ourcurv(\calN_i)\geq\ourcurv(\calN\setminus\{i\})=f(a_i) - f(a_i\,|\,\emptyset)= 0$. For the upper bound, 
\begin{align}
    \ourcurv(\calN_i) &\leq\ourcurv(\emptyset)= f(a_i) - f(a_i\,|\,\calA_{\calN\setminus\{i\}}) \\
    &\leq \kappa_f f(a_i), \label{aux0:11}
\end{align}where the first inequality holds since $\ourcurv$ is non-increasing, and the second inequality holds from \cref{eq:curvature}. \qed
\section*{Appendix II}\label{app:th2}
We first prove \Cref{th:additive}, and then present and prove the bounds of \alg when $f$ is submodular or approximately submodular instead of 2nd-order submodular. 

\paragraph{Proof of \Cref{th:additive}} 
We index each agent in $\calN$ per its selecting order in \alg, \ie  agent $i \in \ouragent\triangleq\{1,\dots,|\ouragent|\}$ is the $i$-th agent to select an action during the execution of \alg. If multiple agents select actions simultaneously, then we index them randomly.  
We use also the notation:
\begin{itemize}[leftmargin=3.5mm]
    \item $\calA^{\scenario{RAG}}_\calX \triangleq \{a^\alg_i\}_{i\myin \calX}$ for any $\calX\subseteq \calN$, \ie $\calA^{\scenario{RAG}}_\calX$ is the set of actions selected by the agents in $\calX$.
\end{itemize}Then we have,
{
\begin{align}
    &f(\calA^\opt)\nonumber\\
    &= f(\calA^\opt\cup\calA^\alg) - \sum_{i\in\calN} f(a_{i}^\alg\,|\,\calA^\opt\cup\calA_{[i-1]}^\alg) \label{aux1:1}\\
    &\leq f(\calA^\alg) + \sum_{i\in\calN} f(a_{i}^\opt\,|\,\calA_{\calN_i\cap [i-1]}^\alg) \nonumber\\
    &\quad - (1-\kappa_{f}) \sum_{i\in\calN} f(a_{i}^\alg\,|\,\calA_{\calN_i\cap [i-1]}^\alg) \label{aux1:2}\\
    &\leq f(\calA^\alg) + \kappa_{f} \sum_{i\in\calN} f(a_{i}^\alg\,|\,\calA_{\calN_i\cap [i-1]}^\alg) \label{aux1:3} \\
    &= (1+\kappa_f)f(\oursol) \nonumber \\
    &\quad+ \kappa_{f} \sum_{i\in\calN}\left[ f(a_i^\alg\,|\,\oursol_{\mathcal{N}_i\cap[i-1]}) -f(a_i^\alg\,|\,\oursol_{[i-1]}) \right]\label{aux1:4}\\
    &\leq (1+\kappa_f)f(\oursol) \nonumber \\
    &\quad+ \kappa_{f} \sum_{i\in\calN} \left[f(a_i^\alg) - f(a_i^\alg\,|\,\oursol_{[i-1]\setminus\mathcal{N}_i}) \right]\label{aux1:5} \\
    &\leq (1+\kappa_f)f(\oursol) + \kappa_{f} \sum_{i\in\calN} \underbrace{\left[f(a_i^\alg) - f(a_i^\alg\,|\,\oursol_{\mathcal{N}_i^c})\right]}_{\ourcurv(\calN_i)}, \label{aux1:6}
    \end{align}}where \cref{aux1:1} holds from telescoping the sums; \cref{aux1:2} holds from $f$ being submodular and
    \begin{equation}
        1-\kappa_{f} \leq \frac{f(a_{i}^\alg\,|\,\calA_{\calN\setminus\{i\}}^\alg)}{f(a_{i}^\alg)} \leq \frac{f(a_{i}^\alg\,|\,\calA^\opt\cup\calA_{[i-1]}^\alg)}{f(a_{i}^\alg\,|\,\calA_{\calN_{i}\cap [i-1]}^\alg)},
    \end{equation}per the definition of $\kappa_{f}$; \cref{aux1:3} holds since \alg selects $\singlesol_i$ greedily; \cref{aux1:4} holds from telescoping the sums; \cref{aux1:5} holds from $f$ being 2nd-order submodular; and \cref{aux1:6} holds from $f$ being submodular. Therefore, \cref{eq:thm-3} holds. 

Then, in the fully centralized case where $\calN_i^c=\emptyset, \forall i\in\calN$, we have $\ourcurv(\calN_i) = f(a_i^\alg) - f(a_i^\alg) = 0$. Hence, \cref{eq:thm-1} follows \cref{aux1:6}. 

Finally, in the fully decentralized case where $\calN_i^c=\calN\setminus\{i\}, \forall i\in\calN$, we have
\begin{align}
    \sum_{i\in\calN}\ourcurv(\calN_i) &= \sum_{i\in\calN}f(a_i^\alg) - f(a_i^\alg\,|\,\calA_{\calN\setminus\{i\}}^\alg) \nonumber\\
    &\leq \kappa_f \sum_{i\in\calN}f(a_i^\alg) \label{aux1:11}\\
    &\leq \frac{\kappa_f}{1-\kappa_f}f(\calA^\alg),\label{aux1:12}
\end{align}where \cref{aux1:11} holds from \cref{eq:curvature}, and \cref{aux1:12} holds from \cite[Lemma 2.1]{iyer2013curvature}. Combining \cref{aux1:6,aux1:12}, the lower bound in \cref{eq:thm-5} can be proved. \qed

\paragraph{Suboptimality Bounds of \alg for Submodular or Approximately Submodular $f$} 
We present the bounds in \Cref{cor:non-submodular}. To this end, we use the following definition. 

\begin{definition}[Total curvature~\cite{sviridenko2013optimal,sviridenko2017optimal}] \label{def:total_curvature}
Consider $f\colon 2^\calV\mapsto\mathbb{R}$ is non-decreasing.  Then, $f$'s total curvature is defined as 
\begin{equation}\label{eq:total_curvature}
c_f\triangleq 1-\min_{v\myin\calV}\min_{\calA, \mathcal{B}\,\subseteq\, \calV\setminus \{v\}}\frac{f(\{v\}\cup\calA)-f(\calA)}{f(\{v\}\cup\mathcal{B})-f(\mathcal{B})}.
\end{equation}
\end{definition}
Similarly to $\kappa_f$, we have $c_f\in [0,1]$. When $f$ is  submodular, then $c_f=\kappa_f$. 
Generally, if  $c_f=0$, then $f$ is modular, while if $c_f=1$, then eq.~\eqref{eq:total_curvature} implies the assumption that $f$ is non-decreasing.
In~\cite{lehmann2006combinatorial}, any monotone $f$ with total curvature $c_f$ is called $c_f$-submodular, as repeated below.\footnote{{Lehmann et al.~\cite{lehmann2006combinatorial}	defined $c_f$-submodularity by considering in eq.~\eqref{eq:total_curvature}  $\calA\subseteq\calB$ instead of $\calA\subseteq \calV$.  Generally, non-submodular but monotone functions have been referred to as \textit{approximately} or \textit{weakly} submodular~\cite{krause2010submodular,elenberg2016restricted}, names that have also been adopted for the definition of $c_f$ in~\cite{lehmann2006combinatorial}, e.g., in~\cite{chamon2016near,guo2019actuator}.}} 

\begin{theorem}[Suboptimality Bounds for Submodular and Approximately Submodular Functions]\label{cor:non-submodular}
If $f$ is not 2nd-order submodular, $\oursol$ enjoys the following approximation bounds:
\begin{itemize}[leftmargin=*]
    \item if $f$ is non-decreasing submodular, then
    \begin{align}\label{eq:th-4-1}
        \frac{f(\,\oursol\,)}{f(\,\calA^\opt\,)} \geq
        \begin{cases} 
        \frac{1}{1+\kappa_f}, & \calG \text{ is {fully centralized}}, \\
        1-\kappa_f, & \calG \text{ is {not fully centralized}},
        \end{cases} 
    \end{align}
    \item if $f$ is non-decreasing $c_f$-submodular, then
        \begin{align}\label{eq:th-4-1}
            \frac{f(\,\oursol\,)}{f(\,\calA^\opt\,)} \geq
            \begin{cases} 
            \frac{1-c_f}{1+c_f-c_f^2}, & \calG \text{ is {fully centralized}}, \\
            (1-c_f)^2, & \calG \text{ is {not fully centralized}}.
            \end{cases} 
        \end{align}
\end{itemize}
\end{theorem}

\paragraph{Proof of \Cref{cor:non-submodular}} We present the proof separately for each case. First, when $f$ is submodular with $\calG$ being fully centralized, \alg has the same $1/(1+\kappa_f)$ bound as in \Cref{th:additive}, following from \cref{aux1:4}.

When $f$ is submodular, and $\calG$ is not fully centralized, \alg provides the same bound as the lower bound in \cref{eq:thm-5}:
{\begin{align}
    f(\calA^\opt) &\leq f(\calA^\alg) + \kappa_{f} \sum_{i\in\calN} f(a_{i}^\alg\,|\,\calA_{\calN_i\cap [i-1]}^\alg) \nonumber \\
    &\leq f(\calA^\alg) + \frac{\kappa_f}{1-\kappa_f} \sum_{i\in\calN} f(a_{i}^\alg\,|\,\calA_{[i-1]}^\alg) \label{aux1:13} \\
    &= \left(1+\frac{\kappa_f}{1-\kappa_f}\right)f(\oursol), \label{aux1:14}
    \end{align}}where \cref{aux1:13} holds from \cref{aux1:3} and the definition of $\kappa_{f}$. 

When $f$ is $c_f$-submodular, and $\calG$ is fully centralized, 
{
\begin{align}
   &\hspace{-1mm}f(\calA^\opt)\nonumber\\
    &\hspace{-1mm}= f(\calA^\opt\cup\calA^\alg) - \sum_{i\in\calN} f(a_{i}^\alg\,|\,\calA^\opt\cup\calA_{[i-1]}^\alg) \nonumber\\
    &\hspace{-1mm}\leq f(\calA^\alg) + \frac{1}{1-c_f}\sum_{i\in\calN} f(a_{i}^\opt\,|\,\calA_{\calN_i\cap [i-1]}^\alg) \nonumber\\
    &\hspace{-1mm}\quad - (1-c_{f}) \sum_{i\in\calN} f(a_{i}^\alg\,|\,\calA_{\calN_i\cap [i-1]}^\alg) \label{aux2:10}\\
    &\hspace{-1mm}\leq f(\calA^\alg) + \left[\frac{1}{1-c_f} - (1-c_f)\right]\sum_{i\in\calN} f(a_{i}^\alg\,|\,\calA_{\calN_i\cap [i-1]}^\alg) \label{aux2:11} \\
    &\hspace{-1mm}= \left(c_f + \frac{1}{1-c_f}\right)f(\calA^\alg),\label{aux2:12}
\end{align}
}where 
\cref{aux2:10} holds from \cref{aux1:1} and \Cref{def:total_curvature}, \cref{aux2:11} holds since \alg selects $\singlesol_i$ greedily, and \cref{aux2:12} holds since $\calN_i\cap [i-1]=[i-1]$ for $\calG$ being fully centralized. 

When $f$ is $c_f$-submodular, and $\calG$ is not fully centralized,
{
\begin{align}
   &f(\calA^\opt)\nonumber\\
    &\leq f(\calA^\alg) + \left[\frac{1}{1-c_f} - (1-c_f)\right]\sum_{i\in\calN} f(a_{i}^\alg\,|\,\calA_{\calN_i\cap [i-1]}^\alg) \nonumber\\
    &\leq f(\calA^\alg) + \left[\frac{1}{(1-c_f)^2} - 1\right]\sum_{i\in\calN} f(a_{i}^\alg\,|\,\calA_{[i-1]}^\alg) \label{aux2:16}\\
    &= \frac{1}{(1-c_f)^2} f(\calA^\alg),\label{aux2:17}
\end{align}}where \cref{aux2:16} holds from \cref{aux2:11} and \Cref{def:total_curvature}.
\qed 
\section*{Appendix III}\label{app:3}
We provide the proofs regarding the a posteriori suboptimality bound of \alg.  
\paragraph{Proof of \Cref{th:posterior}} 
Since $\calI_i = \calN_i\cap [i-1]$, \cref{eq:posterior} follows \cref{aux1:4}, and thus \Cref{th:posterior} is proved. Notice that $f$ only needs to be submodular instead of 2nd-order submodular for \Cref{th:posterior} to hold. \qed

\paragraph{Proof of \Cref{th:posterior-submodular}} 
Let $\delta_i(\calI_i)\triangleq f(a_i^{(\calI_i)}\,|\,\calA_{\calI_i})$, where  
$\calA_{\calI_i}$ is the actions selected by $\calI_i$ per \alg, and $a_i^{(\calI_i)}$ is the action selected by $i$ per \alg, \ie greedily. 
Therefore, proving $\delta_i(\calI_i)$ is non-increasing and approximately supermodular in $\calI_i, \forall i\in\calN$, will be sufficient in proving \Cref{th:posterior-submodular}. 

We start with the non-increasing property by proving $\delta_i$ is non-increasing. 
For disjoint sets $\calB_1, \calB_2\subseteq\calN\setminus\{i\}$, we have $\calA_{\calB_1}\subseteq \calA_{\calB_1\cup\calB_2}$
and, thus,
\begin{align}
    \delta_i(\calB_1) &= f(a_i^{(\calB_1)}\,|\,\calA_{\calB_1}) \geq f(a_i^{(\calB_1\cup\calB_2)}\,|\,\calA_{\calB_1}) \label{aux2:20}\\
    &\geq f(a_i^{(\calB_1\cup\calB_2)}\,|\,\calA_{\calB_1\cup\calB_2}) = \delta_i(\calB_1\cup\calB_2), \label{aux2:21}
\end{align}where \cref{aux2:20} holds since \alg selects $a_i^{(\calB_1)}$ greedily given $\calA_{\calB_1}$, and \cref{aux2:21} holds since $\calA_{\calB_1}\subseteq \calA_{\calB_1\cup\calB_2}$. 

To prove the approximate supermodularity of $\delta_i$, we will first prove another function $\delta_i'$ is supermodular, then show that $\delta_i(\calS)\leq\delta_i'(\calS)\leq\delta_i(\calS)+\epsilon, \forall\calS\subseteq\calN\setminus\{i\}$~\cite{horel2016maximization}. In particular, let us define $\delta_i'(\calI_i)\triangleq f(a\,|\,\calA_{\calI_i})$, where $\calA_{\calI_i}$ is the actions selected per \alg by $\calI_i$ as in the definition of $\delta_i$, but $a\in\calV_i$ is an arbitrary fixed action. Consider robot set $\calS\subseteq\calN\setminus\{i\}$ other than $\calB_1, \calB_2$, then
\begin{align}
    &\delta_i'(\calS\,|\,\calB_1) - \delta_i'(\calS\,|\,\calB_1\cup\calB_2) \nonumber\\
    &\hspace{-1mm}=\delta_i'(\calS\cup\calB_1) - \delta_i'(\calB_1) - \delta_i'(\calS\cup\calB_1\cup\calB_2) + \delta_i'(\calB_1\cup\calB_2) \nonumber\\
    &\hspace{-1mm}= f(a\,|\,\calA_{\calS\cup\calB_1}) - f(a\,|\,\calA_{\calB_1}) \nonumber\\
    &\hspace{-1mm}\quad- f(a\,|\,\calA_{\calS\cup\calB_1\cup\calB_2}) + f(a\,|\,\calA_{\calB_1\cup\calB_2}) \leq 0,
\end{align}where the inequality holds since $f$ is 2nd-order submodular. Hence, $\delta_i'$ is supermodular. 
Then, we have, 
\begin{align}
    \delta_i(\calS) - \delta_i'(\calS) &= f(a_i^{(\calS)}\,|\,\calA_{\calS}) - f(a\,|\,\calA_{\calS}) \nonumber \\
    &\leq f(a_i^{(\calS)}) - (1-\kappa_f)f(a),
\end{align}which holds from the submodularity of $f$ and \cref{eq:curvature}:
\begin{equation}
    1-\kappa_f \leq \frac{f(a\,|\,\calA_{\calN\setminus\{i\}})}{f(a)} \leq \frac{f(a\,|\,\calA_{\calS})}{f(a)}.
\end{equation}Also, $\delta_i(\calS) - \delta_i'(\calS) \geq f(a_i^{(\calS)}\,|\,\calA_{\calS}) - f(a_i^{(\calS)}\,|\,\calA_{\calS})=0$ since \alg selects $a_i^{(\calS)}$ greedily. All in all, $\delta_i'(\calS)\leq\delta_i(\calS)\leq\delta_i'(\calS)+\epsilon$ holds true with \begin{equation}\epsilon=\min_{a_2\myin\calV_i}\max_{a_1\myin\calV_i} \;\left[f(a_1)-(1-\kappa_f)f(a_2)\right],\end{equation}
that is, $\delta_i$ is approximately supermodular. \qed 
\section*{Appendix IV}\label{app:4}
We provide the proof of the communication time of the algorithm in~\cite{gharesifard2017distributed,grimsman2019impact}: 
the worst-case communication time of the two methods occurs when, for example, if the informational DAG $\calG'$ is complete, \ie each agent $i$ requires information from $[i-1]$, then
\begin{itemize}[leftmargin=3.5mm]
    \item for undirected graphs $\calG$, \eg when agent $1$ locates in the center of a line graph and the rest are ordered alternately extending outward from the center to both ends of the line (\eg $5\leftrightarrow 3\leftrightarrow 1\leftrightarrow 2\leftrightarrow 4\leftrightarrow 6$), which leads to $\sum_{i=1}^{|\calN|-1}\tau_c\times i=1/2\,\tau_c\,|\calN|\,(|\calN|-1)$;
    \item for directed graphs $\calG$, \eg when $\calG$ is a one-directional circlic graph yet the agents' order increases in the other direction (\Cref{fig:directed-cyclic-graph}), and every agent $i$ needs to traverse all other agents to send information to $i+1$, which leads to $\sum_{i=1}^{|\calN|-1}\tau_c\times (|\calN|-2)=\tau_c\,(|\calN|-1)(|\calN|-2)$.
\end{itemize}Therefore, the worst-case communication time is $O(\tau_c\,|\calN|^2)$ for both undirected and directed $\calG$. 

\begin{figure}[h]
\captionsetup{font=footnotesize}
\begin{center}
\begin{tikzpicture}
  \node[draw, circle] (1) at (0,-.3) {1};
  \node[draw, circle] (2) at (-3,-1) {2};
  \node[draw, circle] (3) at (-1,-1) {3};
  \node[draw, circle] (4) at (1,-1) {4};
  \node[draw, circle] (5) at (3,-1) {5};

  \draw[->, thick] (1) to[bend right=-13] (5);  
  \draw[->, thick] (2) to[bend right=-13] (1);
  \draw[->, thick] (3) -- (2);
  \draw[->, thick] (4) -- (3);
  \draw[->, thick] (5) -- (4);  
\end{tikzpicture}
\end{center}
\caption{Example of a directed communication graph with the worst-case agent order where each message needs to traverse $|\calN|-2$ edges.
}\label{fig:directed-cyclic-graph}
\end{figure} 

\bibliographystyle{IEEEtran}
\bibliography{references}

\begin{IEEEbiography}[{\includegraphics[width=1in,height=1.25in,keepaspectratio]{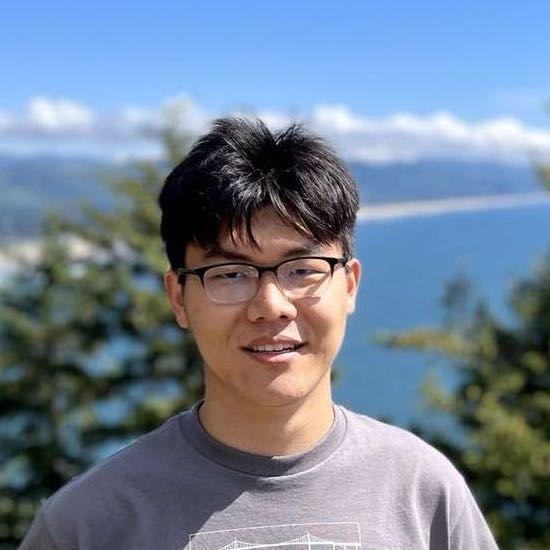}}]{Zirui Xu} (Graduate Student Member, IEEE) received the B.Eng. degree in Automation from Northeastern University, Shenyang, China, in 2018, and the M.S. degree in Electrical and Computer Engineering from Georgia Institute of Technology, Atlanta, GA, USA, in 2020. He is currently pursuing the Ph.D. degree with the Department of Aerospace Engineering, University of Michigan, Ann Arbor, MI, USA. His research focuses on theories and algorithms for scalable and reliable coordination of distributed multi-robot systems in resource-constrained, unstructured, and untrustworthy environments. 
\end{IEEEbiography}

\begin{IEEEbiography}[{\includegraphics[width=1in,height=1.25in,keepaspectratio]{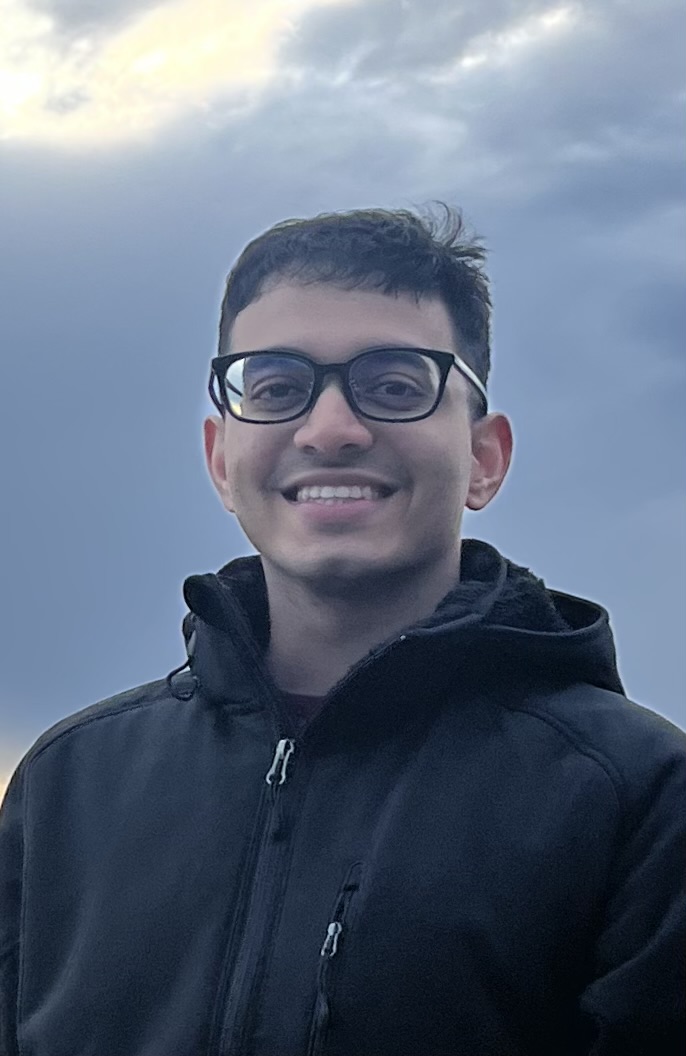}}]{Sandilya Sai Garimella} (Graduate Student Member, IEEE) received a B.S.E. degree in Mechanical Engineering from the University of Michigan, Ann Arbor, MI, USA, in 2022, and a M.S.E. degree in Robotics also from the University of Michigan, Ann Arbor, MI, USA, in 2024. He is currently pursuing a Ph.D. degree in Robotics with the Institute for Robotics and Intelligent Machines (IRIM) at Georgia Institute of Technology, Atlanta, GA, USA. His research interests lie in multi-robot autonomy, with a focus on algorithms for spatial perception, active information acquisition, and efficient multi-agent decision-making. 
\end{IEEEbiography}

\begin{IEEEbiography}[{\includegraphics[width=1in,height=1.25in,keepaspectratio]{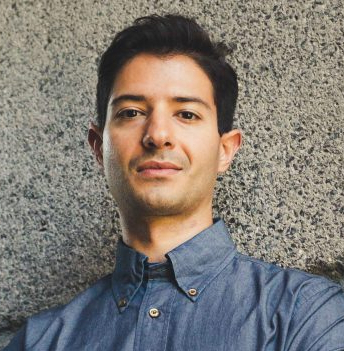}}]
{Vasileios Tzoumas} (Senior Member, IEEE) received his Ph.D. in Electrical and Systems Engineering at the University of Pennsylvania (2018). He holds a Master of Arts in Statistics from the Wharton School of Business at the University of Pennsylvania (2016), a Master of Science in Electrical Engineering from the University of Pennsylvania (2016), and a diploma in Electrical and Computer Engineering from the National Technical University of Athens (2012). Vasileios is an Assistant Professor in the Department of Aerospace Engineering, University of Michigan, Ann Arbor. Previously, he was at the Massachusetts Institute of Technology (MIT), in the Department of Aeronautics and Astronautics, and in the Laboratory for Information and Decision Systems (LIDS), where he was a research scientist (2019-2020) and a post-doctoral associate (2018-2019). Vasileios works on algorithms and innovative hardware for scalable and reliable cyber-physical systems in resource-constrained, uncertain, and contested environments via resource-aware decision-making, online learning, and resilient adaptation. Vasileios is a recipient of an NSF CAREER Award, the Best Paper Award in Robot Vision at the 2020 IEEE International Conference on Robotics and Automation (ICRA), an Honorable Mention from the 2020 IEEE Robotics and Automation Letters (RA-L), and was a Best Student Paper Award finalist at the 2017 IEEE Conference in Decision and Control (CDC).
\end{IEEEbiography}

\end{document}